\newif\ifshaphered

\ifdefined\isshaphered
\shapheredtrue
\fi
\documentclass[sigconf, nonacm]{acmart}

\usepackage{algorithm}
\usepackage{algorithmicx}
\usepackage{algpseudocode} 
\usepackage{multirow}
\usepackage{diagbox}
\usepackage{balance}
\newcommand\vldbdoi{XX.XX/XXX.XX}

\usepackage{enumitem}

\newcommand\vldbavailabilityurl{URL_TO_YOUR_ARTIFACTS}
 
\newcommand{\tabincell}[2]{\begin{tabular}{@{}#1@{}}#2\end{tabular}}

\ifshaphered
\usepackage{longtable}
\newcommand\revision[1]{\textcolor{blue}{#1}}

\else
\newcommand\revision[1]{#1}

\fi

\begin{document}
\title{OpBoost: A Vertical Federated Tree Boosting Framework Based on Order-Preserving Desensitization}

	
\author{Xiaochen Li}
\authornote{Both authors contributed equally to this work.}
\affiliation{%
  \institution{Zhejiang University}
}

\author{Yuke Hu}
\authornotemark[1]
\affiliation{%
\institution{Zhejiang University}
}

\author{Weiran Liu}
\affiliation{%
  \institution{Alibaba Group}
}

\author{Hanwen Feng}
\affiliation{%
\institution{Alibaba Group}
}

\author{Li Peng}
\affiliation{%
\institution{Alibaba Group}
}

\author{Yuan Hong}
\affiliation{%
\institution{University of Connecticut}
}

\author{Kui Ren}
\affiliation{%
\institution{Zhejiang University}
}

\author{Zhan Qin}
\authornote{Corresponding author.}
\affiliation{%
\institution{Zhejiang University}
}

\begin{abstract}
Vertical Federated Learning (FL) is a new paradigm that enables users with non-overlapping attributes of the same data samples to jointly train a model without directly sharing the raw data.
Nevertheless, recent works show that it's still not sufficient to prevent privacy leakage from the training process or the trained model.
This paper focuses on studying the privacy-preserving tree boosting algorithms under the vertical FL.
The existing solutions based on cryptography involve heavy computation and communication overhead and are vulnerable to inference attacks.
Although the solution based on Local Differential Privacy (LDP) addresses the above problems, it leads to the low accuracy of the trained model.

This paper explores to improve the accuracy of the widely deployed tree boosting algorithms satisfying differential privacy under vertical FL.
Specifically, we introduce a framework called OpBoost.
Three order-preserving desensitization algorithms satisfying a variant of LDP called distance-based LDP (dLDP) are designed to desensitize the training data.
In particular, we optimize the dLDP definition and study efficient sampling distributions to further improve the accuracy and efficiency of the proposed algorithms.
The proposed algorithms provide a trade-off between the privacy of pairs with large distance and the utility of desensitized values.
Comprehensive evaluations show that OpBoost has a better performance on prediction accuracy of trained models compared with existing LDP approaches on reasonable settings.
\revision{Our code is open source.\footnote{https://github.com/alibaba-edu/mpc4j/tree/main/mpc4j-sml-opboost}}
\end{abstract}

\maketitle

\renewcommand\thefootnote{}\footnote{
Email:\{xiaochenli, yukehu, kuiren, qinzhan@zju.edu.cn, weiran.lwr, fenghanwen.fhw, jerry.pl@alibaba-inc.com, yuan.hong@uconn.edu\}\\
}

\section{Introduction}
Federated Learning (FL) \cite{konevcny2016federated} is an emerging paradigm that enables multiple parties to jointly train a machine learning model without revealing their private data to each other.
According to the way of data partitioning, FL can be classified into two categories: Horizontal FL and Vertical FL \cite{yang2019federated}. 
\textit{Horizontal} FL considers the scenarios where different data samples with the same features are distributed among different parties.
\textit{Vertical} FL works when different parties hold the same data samples with disjoint features. 
Vertically distributed datasets are very common in real-world scenarios.
One typical example of vertical FL is shown in Figure \ref{fig:examp}.
A credit institution collaborates with an E-commerce company and a bank to train a model to predict the labels, i.e., users' credit ratings, based on features, i.e., shopping records and revenue.

Tree boosting algorithms (e.g., GBDT \cite{friedman2001greedy}) are popular supervised ML algorithms that enjoy high efficiency and interpretability.
They are widely used in prediction tasks based on heterogeneous features, i.e., revenue, age, in the scenarios, i.e., credit, price forecast \cite{chen2016xgboost, burges2010ranknet, dorogush2018catboost, ke2017lightgbm}.
Recent work \cite{gorishniy2021revisiting} shows that when the training data mainly includes heterogeneous features, GBDT outperforms state-of-the-art deep learning solutions.
However, most existing tree boosting algorithms are proposed in the centralized setting, which train the model with direct access to the whole training data.
With the increasing public privacy concerns and the promulgation of privacy regulations (e.g., GDPR), this centralized setting limits the widespread deployment of tree boosting algorithms.
Therefore, it becomes an essential problem to design practical privacy-preserving vertical federated tree boosting algorithms.

\begin{figure}[t]
	\centering
	\includegraphics[width=0.49\textwidth]{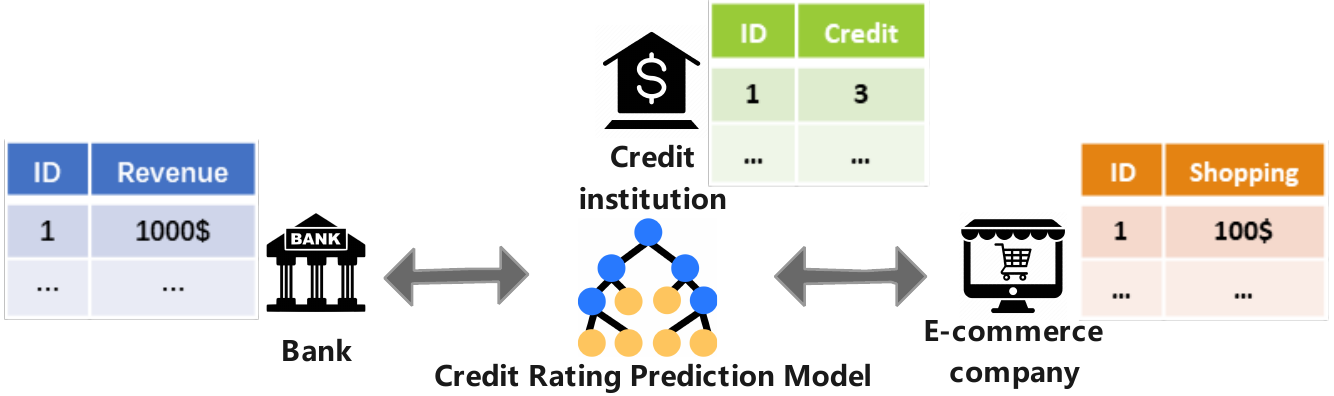}
	\setlength{\belowcaptionskip}{-2ex}
	\caption{An example of Vertical FL.}
	\label{fig:examp}
\end{figure}

To address this problem, there are some solutions based on two different technologies: \emph{Cryptography}, and \emph{Local Differential Privacy (LDP)}.
Framework SecureBoost \cite{cheng2021secureboost} and its subsequent work $\text{VF}^2$Boost \cite{fu2021vf2boost} are based on additive \emph{homomorphic encryption}.
Although well-designed engineering optimization is performed, a large number of homomorphic operations still inevitably cause participants to suffer a prohibitively computational overhead.
Besides, each party shares the true order of their feature values for training the model.
To our knowledge, some attacks have been proposed to use auxiliary information to infer the distribution of original values based on the order of values \cite{boldyreva2009order, durak2016else, bindschaedler2018tao}.
Abspoel et al. \cite{abspoel2021secure} present a simple vertical federated boosting framework based on Multi-Party Computation (MPC) protocols.
Nevertheless, MPC protocols cause a heavy communication overhead.
Although the order of feature values is not visible to any party in MPC protocol, the trained model is not a privacy-preserving model, which is vulnerable to inference attack \cite{song2019membership, naveed2015inference, kellaris2016generic}.
Tian et al. \cite{tian2020federboost} propose FederBoost, which provides \emph{LDP} privacy protection for each party's feature values and can resist all the aforementioned attacks.
FederBoost is shown to be much more efficient than the MPC-based and Encryption-based solutions, which is favorable in real-world applications.
However, the introduction of randomness leads to a serious loss of the relative order information, and results in low accuracy of the trained model.

\textbf{Observation.}
The process of building the decision tree in tree boosting algorithms is to constantly find the split points of the features, and this process only depends on the order of feature values rather than the exact values.
In existing solutions, each party is essentially sending the order of feature values to the party holding the label for training the model.
It's necessary to provide privacy protection for the order of feature values since training the model with the true order of the feature values can leak the original private values.
However, the mechanisms satisfying LDP are usually designed to provide privacy protection for values without order (e.g., enumeration values).
They perturb the private values to achieve the same degree of indistinguishability for any pair of values in the data domain.
Meanwhile, this causes desensitized values to lose too much order information, which seriously reduces the accuracy of the trained model.
In fact, people require different degrees of indistinguishability for pairs of values with different distances.
For example, an employee doesn't mind being revealed that his income is less than his boss, but he minds being known by others to be less than his colleagues.
Therefore, it is more suitable to provide different degrees of indistinguishability for pairs with different distances.
Meanwhile, the relative order of value pairs, especially those that are far apart, can be preserved with a high probability.

Another observation is that the existing distance-based LDP (dLDP) definition has limitations.
There is only one privacy parameter $\epsilon$ in the existing dLDP definition, which allocates privacy budgets based on $l_1$ distance between two values.
Given the total privacy budget $\epsilon$, one might want its private value to be as indistinguishable as possible from its nearby values, and not mind weakening the indistinguishability from the values farther away.
However, the existing dLDP definition cannot achieve this privacy requirement.
Specifically, increasing $\epsilon$ can increase the probability of the desensitized value falling near the true value, but at the same time its distribution near the true value is more concentrated, and vice versa.
If we can increase the probability of the desensitized value falling in a specific area around the true value, but flatten the probability distribution in this area, an optimized output probability distribution of desensitized value can be obtained.

\textbf{Contribution.}
Our contributions are summarized below.

\textbf{\it Proposal of OpBoost:}
We propose a novel framework called OpBoost for privacy-preserving vertical federated tree boosting.
Within the framework, we design three order-preserving desensitization algorithms for desensitizing the training data.
Different from the existing LDP-based solution, the desensitized training data satisfy a variant of LDP called dLDP.
It can preserve more relative order information of desensitized values than LDP while providing privacy protection.
When strong indistinguishability is required for close values, i.e., $\epsilon=0.08$ for value pairs with distance $t=1$, OpBoost can still achieve accuracy close to that without protection for both classification and regression tasks.
For example, for a classification task, OpBoost achieves $60\%$ when no protection is $87\%$, while the LDP-based solution is close to $10\%$.
Meanwhile, OpBoost also retains the advantages of LDP-based solutions over Cryptography-based solutions.
The total communication overhead of each party is about $O(rn\log n)$ bits, whereas $O(rnk\log^2 n)$ is required in the MPC-based solution ($n$, $k$, $r$ are the number of samples, values' bits, and features, respectively).
Moreover, we replace the exponential mechanism with the (bounded) Laplace mechanism to reduce the computational complexity of desensitizing a sensitive value to $O(1)$.

\textbf{\it Optimizing existing dLDP definition:}
We also optimize the existing dLDP definition in order to break through its limitations.
We divide the data domain into several partitions with the length of $\theta$.
Then we introduce two privacy parameters $\epsilon_{prt}$ and $\epsilon_{ner}$ to adjust the probability distribution of desensitized value falling in different partitions and the probability distribution within one partition, respectively.
We prove that the existing definition is just a special case where $\epsilon_{prt}$ and $\epsilon_{ner}$ satisfy a fixed proportional relationship.
We can always get higher accuracy than the existing dLDP under the same privacy guarantee by adjusting the ratio of $\epsilon_{prt}$ and $\epsilon_{ner}$.

\textbf{\it Introducing new order-preserving metrics:}
In addition to quantifying the privacy of the order-preserving desensitization algorithms with dLDP, we also introduce new theoretical and experimental metrics to quantify the order information preserved by desensitized values. 
We define that the proposed order-preserving desensitization algorithms are probabilistic order-preserving in theory.
The probability of any pair of desensitized values preserving the original relative order is at least $\gamma$.
Besides, we introduce the weighted-Kendall coefficient weighted by distance to evaluate the order of desensitized feature values experimentally.

\textbf{\it Comprehensive Evaluation:}
We conduct comprehensive theoretical and experimental evaluations to analyze the performance of OpBoost, including all the designed order-preserving desensitization algorithms.
We evaluate the order preservation of the desensitized values using all introduced metrics.
We also conduct the experiments on public datasets used for \emph{Binary Classification}, \emph{Multiple Classification}, and \emph{Regression} tasks.
Both GBDT and XGBoost are implemented in OpBoost.
The experimental results show that OpBoost achieves the prediction accuracy close to and even higher than plain models, i.e.,1.0003$\times$ improvement over plain model of XGBoost, which is superior to the existing LDP approaches.

\section{Preliminaries}
\subsection{Differential Privacy}
Differential privacy \cite{dwork2008differential} is the \textit{de facto} privacy definition of data disclosure, preventing attempts from learning private information about any individual in a data release. 
In this work, we are interested in \textit{local} differential privacy \cite{kasiviswanathan2011can}, which allows each user to perturb his sensitive data using a randomization mechanism $\mathcal{M}$ such that the perturbed results from different data values will be ``close".

\begin{definition}\label{def:local-differential-privacy} (\emph{Local Differential Privacy, LDP}). 
  An algorithm $\mathcal{M}$ satisfies $\epsilon$-LDP, where $\epsilon \geq 0$, if and only if for any input $v, v' \in \mathbb{D}$, and any output $y \in Range(\mathcal{M})$, we have 
  \[\setlength\abovedisplayskip{0.5ex}
     \setlength\belowdisplayskip{0.5ex}
  \Pr \left[\mathcal{M}(v) = y \right] \leq e^{\epsilon} \Pr \left[\mathcal{M}(v') = y \right].\]
\end{definition}
The parameter $\epsilon$ above is called the \emph{privacy budget}; the smaller $\epsilon$ means stronger privacy protection is provided. 
On the other hand, since all pairs of sensitive data shall satisfy $\epsilon$-privacy guarantee for the same $\epsilon$, it may hide too much information about a dataset, such that utility might be insufficient for certain applications.
The distance-based LDP \cite{alvim2018local, chatzikokolakis2013broadening, he2014blowfish} is proposed to improve the utility, which measures the level of privacy guarantee between any pair of sensitive data based on their distance.
We use $l_1$ distance in this paper, and the definition of distance-based local differential privacy is defined as follows.

\begin{definition}\label{def:dlocal-differential-privacy} (\emph{Distance-based Local Differential Privacy, dLDP}). 
  An algorithm $\mathcal{M}$ satisfies $\epsilon$-dLDP, if and only if for any input $x, x' \in \mathbb{D}$ such that $|x-x'|\le t$, and any output $y \in Range(\mathcal{M})$, we have 
  \[\setlength\abovedisplayskip{0.5ex}
     \setlength\belowdisplayskip{0.5ex}
  \Pr \left[\mathcal{M}(x) = y \right] \leq e^{t\epsilon}\cdot \Pr \left[\mathcal{M}(x') = y \right],\]
\end{definition}
\noindent where $t\epsilon$ controls the level of indistinguishability between outputs of $\mathcal{M}(x)$ and $\mathcal{M}(x')$.
The indistinguishability decreases as the distance $t$ between $x$ and $x'$ increases.

\subsection{Order-Preserving Desensitization Algorithm}
In some application scenarios (e.g., recommender system, range query), the accuracy of the algorithm mainly depends on the order of the dataset. 
It would be desirable that the numerical order of sensitive data is somehow preserved after desensitizing.
A lot of order-preserving desensitization algorithms are proposed in cryptographic studies \cite{kerschbaum2015frequency, kerschbaum2014optimal, boldyreva2009order, agrawal2004order}, in which the order is rigorously preserved after desensitization.
The formal definition of the order-preserving desensitization algorithm is as follows.

\begin{definition}\label{def:order-preserving} (\emph{Order-Preserving Desensitization Algorithm}). 
  Denote $X=x_1,x_2,...,x_n$ $(\forall i. x_i\in \mathbb{N})$ as the sensitive sequence, and $Y=y_1,y_2,...,y_n$ $(\forall i. y_i\in \mathbb{N})$ be the noisy sequence output by a desensitization algorithm $\mathcal{R}$, where $y_i = \mathcal{R}(x_i)$.
  The algorithm $\mathcal{R}$ is order-preserving if and only if the following conditions are satisfied:
  \[{\qquad}\ \ \setlength\abovedisplayskip{0.5ex}
     \setlength\belowdisplayskip{0.5ex}
  \forall i,j.\ x_i> x_j\Rightarrow y_i> y_j, \ \ \;and\]
  \[\setlength\abovedisplayskip{0.5ex}
     \setlength\belowdisplayskip{0.5ex}
  \forall i,j.\ y_i> y_j\Rightarrow x_i\ge x_j.\]
\end{definition}

However, rigorous order itself could be leveraged by attackers to perform attacks (e.g., big-jump attack\cite{boldyreva2009order}, inter-column correlation-based attack \cite{durak2016else}, multinomial attack \cite{bindschaedler2018tao}, inference attack \cite{naveed2015inference, kellaris2016generic}).
These attacks use auxiliary information to estimate the distribution of the original values and then correlate them with the desensitized values based on their order.
Besides, there is a lack of widely accepted cryptography tools to quantify how much privacy is compromised through attacks.
The notion of differential privacy can help with these predicaments.
It provides a rigorous upper bound for information disclosure and turns deterministic output into probabilistic results.
Hence, we extend a relaxed version of the order-preserving notion called \emph{Probabilistic Order-Preserving} in Definition \ref{def:probabilistic-order-preserving}.

\begin{definition}\label{def:probabilistic-order-preserving} (\emph{Probabilistic Order-Preserving Desensitization Algorithm}). 
  Denote $X=x_1,x_2,...,x_n$ $(\forall i. x_i\in \mathbb{N})$ as the sensitive sequence, and $Y=y_1,y_2,...,y_n$ $(\forall i. y_i\in \mathbb{N})$ be the noisy sequence output by a desensitization algorithm $\mathcal{R}$, where $y_i = \mathcal{R}(x_i)$.
  The algorithm $\mathcal{R}$ is probabilistic order-preserving if and only if the following conditions are satisfied:
  \[\setlength\abovedisplayskip{0.5ex}
     \setlength\belowdisplayskip{0.5ex}
  \forall i,j.\ x_i> x_j\Rightarrow Pr[y_i> y_j]\ge \gamma(t),\ where\ \gamma(t)\in[0, 1], |x_i-x_j|\le t.\]
\end{definition}
\ifshaphered
\emph{(Reviewer2:D3)}
\else
\fi
\revision{
Here, $\gamma$ is a function related to the distance between $x_i$ and $x_j$.
The definition satisfies rigorous order-preserving desensitization in Definition \ref{def:order-preserving} when $\gamma(t)=1$ for any $t$.
The algorithms satisfying probabilistic order-preserving preserve the relative order of partial pairs of values rather than all pairs with $\gamma(t)<1$.}
Specifically, the probabilistic order-preserving desensitization algorithms can be achieved by adding carefully selected random noise to sensitive values. 
Meanwhile, randomness can provide provable privacy guarantee to resist all aforementioned attacks based on auxiliary information.
All the proposed desensitization algorithms are probabilistic order-preserving.
Moreover, we prove that these algorithms all satisfy dLDP.

\subsection{Gradient Tree Boosting}
The term "gradient tree boosting" originates from the paper by Friedman et al. \cite{friedman2000additive}.
Each iteration of training involves incrementally adjusting the gradient to fit the residual, with the goal of minimizing the loss function.
There are some gradient tree boosting algorithms have been widely used such as GBDT \cite{friedman2001greedy}, XGBoost \cite{chen2016xgboost}, where XGBoost is an efficient implementation of GBDT.

We analyze the process of XGBoost building decision trees to help understand that only the order of features' values is necessary during the process of gradient tree boosting.
The algorithm continually finds the split point with the greatest gain after splitting.
We denote that $I_L$ and $I_R$ are the sample sets of left and right nodes after splitting, $g_i$ and $h_i$ are gradients, $\lambda$ and $\omega$ are regularization parameters, the gain of the split is given by
\[\setlength\abovedisplayskip{0.5ex}
     \setlength\belowdisplayskip{0.5ex}
G_{split}=\frac{1}{2}[\frac{(\sum_{i\in I_L}g_i)^2}{\sum_{i\in I_L}h_i+\lambda}+\frac{(\sum_{i\in I_R}g_i)^2}{\sum_{i\in I_R}h_i+\lambda}-\frac{(\sum_{i\in I}g_i)^2}{\sum_{i\in I}h_i+\lambda}]-\omega.\]

Note that all the variables we need for calculating $G_{split}$ can be derived only from the order of features' values.
Thus we can build the boosting tree without knowing the exact values of each feature.
We can decide which node should a sample fall in based on this tree, but can not know the value that this node represents, which made it a "partial tree".
The accurate predictions can be achieved with the assistance of the parties holding corresponding features.

\begin{figure*}[t]
  \centering
  \includegraphics[width=0.9\textwidth]{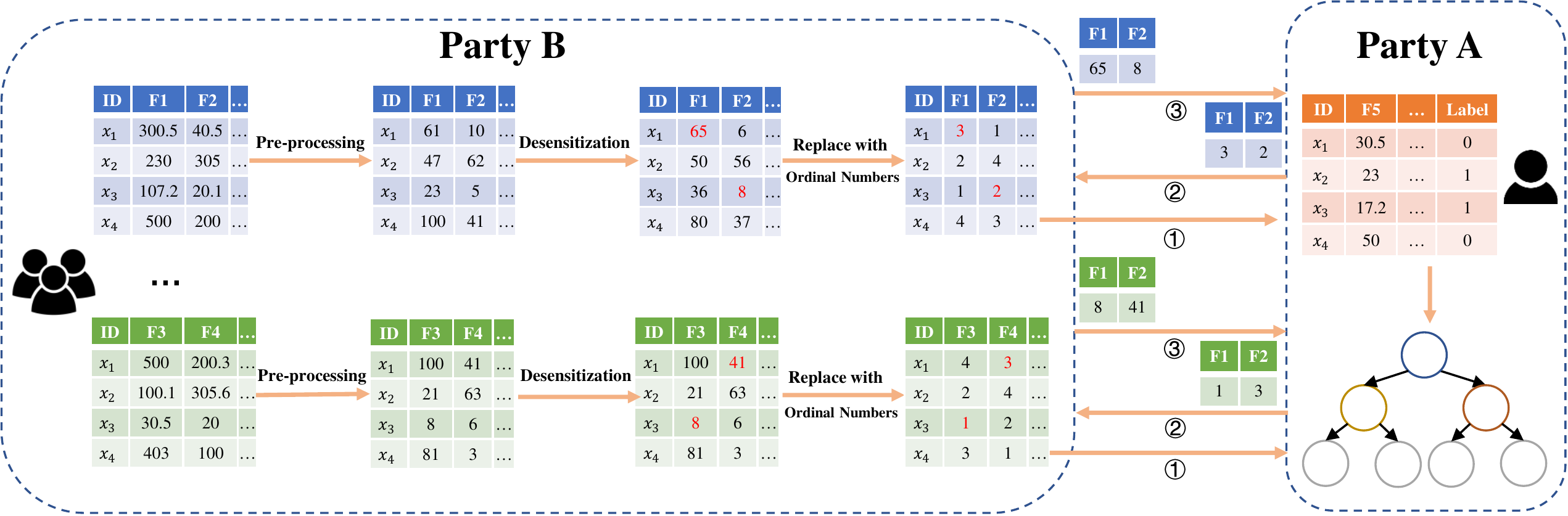}
  \caption{Training process of OpBoost.}
  \label{fig:train-process}
\end{figure*}

\section{System Overview}
\subsection{Architecture}
The training dataset is vertically partitioned and is distributed among different users' devices.
Each user holds different features of samples but overlapping sample IDs and only one user holds labels.
Since the label is essential for supervised learning, the user holding labels is generally the central node responsible for aggregating information and updating the model.
Therefore, our framework focuses on guiding the information exchange between other users and the user holding labels, rather than sharing labels among all users.
For the sake of simplicity, we divide users into two parties.

\noindent\textbf{Party A.}
Party A refers to the user who holds the label of the training samples.
It may also hold several features.
In the training process, Party A acts as a central server to exchange information with Party B and train the model.
The trained model is only stored in Party A, and Party A is responsible for using it to predict the new samples.

\noindent\textbf{Party B.}
We define Party B as the set of users who only hold several features of the samples.
Party B acts as a client that participates in training by exchanging necessary information with Party A.

\vspace{-1em}
\subsection{Execution Workflow}
In the following, we describe the process of training decision trees with OpBoost in detail, and also explain how to predict with these trained decision trees.

\noindent\textbf{Training Process.}
The overall process can be summarized into three steps, which are shown in Figure \ref{fig:train-process}.

First, Party B desensitizes the local features before communicating with Party A, which is specified as follows.

\begin{itemize}[leftmargin=*]
  \item \emph{Pre-processing for Feature Values.}
  \ifshaphered
  (\emph{Reviewer1:D2.a})
  \else
  \fi
  \revision{
  We focus on numerical features and categorical features that have natural ordering between categories.
  The categorical features with no distance between values are not included, and the tree boosting algorithms handle them differently.
  These features are usually encoded by one-hot encoding, and the encoded values can be desensitized by existing LDP mechanisms \cite{wang2017locally, erlingsson2014rappor}.
  As the ordinal categorical values can be mapped to discrete numerical values, w.l.o.g., we assume that all features are numerical values, i.e., continuous or discrete numerical values.
  Besides, since the specific values do not affect the structure of the decision tree, it suffices to remap features coming with diverse distance metrics to a unified discrete data domain for the subsequent distance-based privacy-preserving algorithm to work with.
  }
  \item \emph{Desensitize Values with Order-preserving Desensitization Algorithm.}
  We design several order-preserving desensitization algorithms that satisfy dLDP, and give guidance to help Party B choose algorithm to desensitize features' values.
  \item \emph{Replace the values with serial numbers.}
  Party B replaces all the desensitized values of features with their corresponding ordinal numbers and then sends them to Party A.  
\end{itemize}

Second, Party A finds the best split points over the features after collecting all features' information from Party B.
Specifically, Party A does not know the values of split points for features stored in Party B.
It records the ordinal numbers as the split points.

Finally, Party A sends all order numbers of split points to Party B to get their desensitized values.
Then Party B sends the specific desensitized values of corresponding split points back to Party A, and Party A updates the tree models.

\noindent\textbf{Prediction.}
After the above training steps, Party A can obtain a complete decision tree model for predicting new samples.
Party A can independently predict the new samples stored locally (non-private), or continue to cooperate with Party B to predict new samples (private).
All the new samples need to do the same pre-processing as the training samples before desensitization or being input into the model.

Note that Party B is not required to be online all the time in both training and prediction procedures.
It can go offline after sending all features' information and values of split points to Party A.
In addition, Party A can utilize the trained decision tree to independently perform the prediction tasks with non-private samples or desensitized private samples.

\begin{table}[t]
  \small
  \centering
  \begin{tabular}{|c|c|}
      \hline
      \textbf{Variable}&\textbf{Description}\\
      \hline
      $\mathbb{D}$&Finite and numerical input data domain\\
      \hline
      $\mathbb{D}_{\bot}$&Finite and discrete data domain after mapping\\
      \hline
      $L/R$&Lower/Upper bound of $\mathbb{D}_{\bot}$\\
      \hline
      $t$&Distance between values in $\mathbb{D}_{\bot}$\\
      \hline
      $\theta$&Length of a partition\\
      \hline
      $\mathcal{P}_m$& $m^{\text{th}}$ partition of $\mathbb{D}_{\bot}$\\
      \hline
      $\epsilon_{prt}/\epsilon_{ner}$&\tabincell{c}{Parameter for Adjusting the privacy budget\\ between different partitions/within one partition}\\
      \hline
      $\alpha$&Ratio of $\epsilon_{prt}$ and $\theta \epsilon_{ner}$\\
      \hline
      $\gamma$&Lower bound of order-preserving probability\\
      \hline
      
  \end{tabular}
  \setlength{\abovecaptionskip}{-2ex} 
  \setlength{\belowcaptionskip}{0ex}
  \caption{Important Notations}
  \label{tab:notations}
  \vspace{-0.5em}
\end{table}

\section{Proposed Algorithms}
\subsection{Pre-Processing for Feature Values}
\label{sec:preprocess}
Since the training samples come from multiple parties, the tree boosting algorithms usually preprocess the values of all features before training, i.e., fulling the missing values, handling wrong values.
In addition, we present an additional preprocessing step to improve the privacy and utility of desensitization algorithms.
Firstly, there are some existing works that propose that implementing differential privacy mechanisms over floating-point numerical values is vulnerable to privacy attacks \cite{mironov2012significance, ilvento2020implementing}.
Secondly, note that the privacy guarantee provided by OpBoost satisfies distance-based LDP.
It's necessary to normalize the values of different features in a unified distance unit.
To address these issues, we map numerical values of all different features into a unified discrete value domain.
We show the details in the following.

We remap discrete numerical values by $\emph{Mapping Function}$ $\mathcal{B}$ to a unified discrete value domain $\mathbb{D}_{\bot}$.
Denote $\mathcal{X}_{c}\in\mathbb{D}$ as the set of numerical values of a feature.
Party B maps each value $\mathcal{X}_c^i$ as follows
\[\setlength\abovedisplayskip{0.5ex}
     \setlength\belowdisplayskip{0.5ex}
\mathcal{X}_{int}^i=\lceil L+\frac{\mathcal{X}_{int}^i-lower}{upper-lower}\cdot (R-L)\rceil.\]
where $lower$ and $upper$ are lower bound and upper bound of $\mathbb{D}$, $L$ and $R$ are lower bound and upper bound of $\mathbb{D}_{\bot}$, respectively.
The larger the domain $\mathbb{D}_{\bot}$ is, the more original relative orders the mapped values preserved.
We take an example to explain why we need to map values of different features into a unified value domain.
The  values of \emph{age} are usually in the range of $(0,\ 100]$ years, while values of \emph{salary} are usually in the range of $(0,\ 100,000]$ dollars/year.  
It is easy to see that the sensitivity of changing $10$ years old to $30$ years old is not the same as that from $2000$ dollars/year to $2020$ dollars/year, though the differences are the same. 
Therefore, it is necessary to map all values with different meanings to the same value domain to facilitate evaluating the privacy guarantee and the utility of the desensitized training dataset.

\begin{figure*}[t]
  \centering
  \includegraphics[width=0.8\textwidth]{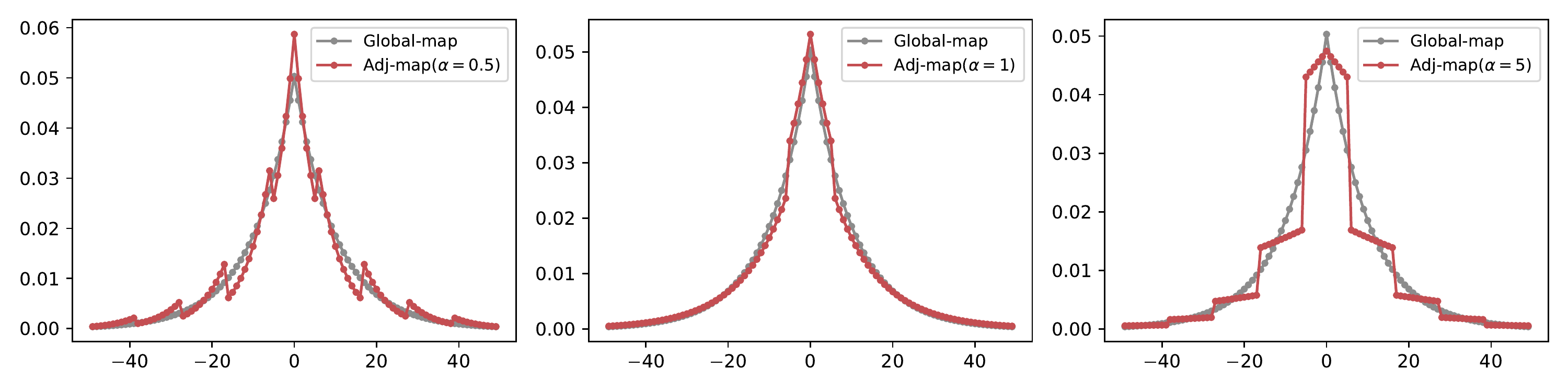}
  \caption{Adjust the ratio of $\epsilon_{prt}$ and $\epsilon_{ner}$, where $\mathbb{D}_{\bot}$ is limited to $[-50,50]$, $\epsilon=0.1$, $\theta=10$.
           The probability distribution of output values of Global-map achieved by exponential mechanism and Adj-map are almost fitted when $\alpha=1$ ($\epsilon_{prt}=\theta\epsilon_{ner}$).}
  \label{fig:distribution}
\end{figure*}

\subsection{Order-preserving Desensitization Algorithms}
\label{sec:algorithms}
Since preserving all order information of values is vulnerable to inference attacks, it is necessary to introduce randomness into the orders of desensitized values.
Although desensitizing values with LDP mechanisms can resist privacy attacks, a lot of order information is lost, leading to the trained model's poor performance.
Instead of achieving the same indistinguishability of any pair of values as LDP, the distance-based LDP (dLDP) provides different indistinguishability for pairs of values with different distances.
The pairs of values with smaller distance is harder to be distinguished.
While in the training of decision tree, the inverse of a pair of values with larger distance has a more significant impact on the trained model (split point of decision tree changes with higher probability).
Besides, the public does not mind the indistinguishability between values with large distance in practice  (e.g., an office worker does not mind the fact that he spends less time doing sports than an athlete being revealed).
Therefore, the dLDP definition is more suitable for order-preserving desensitization than LDP.

In this subsection, we first take a variant of exponential mechanism \cite{gursoy2019secure, chowdhury2020intertwining} satisfying dLDP as the preliminary algorithm, called Global-map.
The probability that each value in the data domain is output as a desensitized value is inversely proportional to its distance from the sensitive value.
We then do an in-depth study of Global-map from the definition of dLDP, and propose optimization algorithms for it.
Note that the input values of all algorithms are pre-processed by the algorithms proposed in Section \ref{sec:preprocess}.
We denote $\mathbb{D}_\bot$ as the data domain of both input values and output values, and denote $t$ as the distance of a pair of input values.
The important notations are shown in Table \ref{tab:notations}.

\begin{algorithm}[t]
  \caption{Global-map}
  \label{alg:base}
  {\small{
  \begin{algorithmic}[1]
        \Require {$x\in\mathbb{D}_\bot$, parameter $\epsilon>0$}
        \Ensure {$o\in\mathbb{D}_{\bot}$}
        \For {$i\in \mathbb{D}_{\bot}$}
        \State {$p_{x,i}=Pr[o=i]=\frac{e^{-|x-i|\cdot\epsilon/2}}{\sum_{j\in\mathbb{D}_{\bot}}e^{-|x-j| \cdot\epsilon/2}}$}
        \EndFor
        \State {Sample $o\sim p_x=\{p_{x,L},\cdots,p_{x,R}\}$}\\
        \Return {$o$}
  \end{algorithmic}}}
\end{algorithm}

\noindent \textbf{Global-Map:}
Global-map is built upon a variant of exponential mechanism, which assigns the output probability to each value according to a score function. 
The score function can be defined as the distance between the value and sensitive value so that a value will be desensitized to a nearer value with a higher probability.
The details of Global-map are shown in Algorithm \ref{alg:base}.
The privacy guarantee provided by Global-map is shown in Theorem \ref{the:base}.

\begin{theorem}
  \label{the:base}
  Global-map provides $\epsilon$-dLDP privacy guarantee for any pair of values $x,x'\in\mathbb{D}_\bot$, where $|x'-x|\le t$, and $t, \epsilon>0$.
\end{theorem}
\begin{proof}
  The proof is presented in Appendix \ref{app:base}.
\end{proof}

Then we analyze the utility of Global-map.
We theoretically analyze the order-preserving degree of values after being desensitized by Global-map.
According to the probabilistic order-preserving definition (Definition \ref{def:probabilistic-order-preserving}), we calculate the order-preserving probability $\gamma$ of any pair of desensitized values.
The result is shown in Theorem \ref{the:baseuti}.

\begin{theorem}
  \label{the:baseuti}
  Global-map is a probabilistic order-preserving desensitization algorithm with $\gamma(t)\ge1-\frac{(1-q^2)\cdot t+1}{(1+q-q^{t+1}-q^{|\mathbb{D}_{\bot}|-t})(1+q)}\cdot q^t$, where $q=e^{-\epsilon/2}$.
\end{theorem}

\begin{proof}
  The proof is presented in Appendix \ref{app:uti-global-map}.
\end{proof}

The LDP mechanism Generalized Random Response (GRR) \cite{wang2017locally} used in \cite{tian2020federboost} can also be regarded as an probabilistic order-preserving desensitization algorithm.
To show the advantage of dLDP mechanism in ordinal preservation compared with LDP mechanism, we also calculate the order-preserving probability $\gamma$ of GRR.
\revision{
The result is shown in Theorem \ref{the:ldputi}. In addition, we also provide an intuitive comparison and detailed theoretical analysis in Appendix \ref{app:_theo_OPP}.
}
\begin{theorem}
  \label{the:ldputi}
  GRR is a probabilistic order-preserving desensitization algorithm with $\gamma(t)=p_1^2+p_1p_2\cdot(|\mathbb{D}_{\bot}|-3)+p_2^2\cdot(\frac{1}{2}|\mathbb{D}_{\bot}|(|\mathbb{D}_{\bot}|-3)+2) + p_2(p_1-p_2)t$, where $p_1=\frac{e^{\epsilon}}{|\mathbb{D}_{\bot}|+e^{\epsilon}-1}$, $p_2=\frac{1}{|\mathbb{D}_{\bot}|+e^{\epsilon}-1}$.
\end{theorem}

\begin{proof}
  The proof is presented in Appendix \ref{app:uti-grr}.
\end{proof}
\revision{
Here, $p_2(p_1-p_2)t<1/((\mathbb{D}_{\bot}|+e^\epsilon-1)(\frac{1}{e^\epsilon-1}+\frac{1}{(|\mathbb{D}_{\bot}|}))$, is usually small enough to be negligible, especially when $|\mathbb{D}_{\bot}|$ is large.
}

\noindent \textbf{Adj-Map:}
We analyzed that it is more reasonable to assign the probability of mapping other values in domain $\mathbb{D_{\bot}}$ based on distance when desensitizing an ordered numerical value.
However, we found that the existing dLDP definition allocates privacy budgets based on $l_1$ distance between two values is not sufficient for satisfying all privacy requirements of feature values with rich semantics.
Taking feature \emph{age} as an example, people usually think that ages in the partition of $[1,30]$ are young, in the partition of $[30,60]$ are middle-aged, and those over $60$ years old are elderly.
Therefore, although the $l_1$ distances of value pairs $\{20, 25\}$ and $\{28, 33\}$ are all $5$, $28$ and $33$ span two partitions.
In other words, if a person's age is desensitized from $28$ to $33$, his identity changes from a young to a middle-aged person, while desensitization from $20$ to $25$ will not.
If people don't care about which partition their ages belong to, but are more concerned about the indistinguishability of values from the same partition.
The privacy definition should allocate more privacy budgets to value pairs located in different partitions, and allocate less privacy budget for values pairs located in the same partition.
Obviously, the existing dLDP definition cannot achieve the above privacy budget allocation since $l_1$ distance cannot distinguish whether value pairs are in different partitions.

To address the above problem, we replace the privacy budget $\epsilon$ defined by existing dLDP with two parameters $\epsilon_{prt}$ and $\epsilon_{ner}$ to adjust the privacy budget between different partitions and within the same partition, respectively.
We give a new distance-based local differential privacy definition \emph{partition-dLDP}, which defines the protection strength based on the $l_1$ distance of value pairs and the number of partitions between them.

\begin{algorithm}[t]
  \caption{Mapping Partition of Adj-Map}
  \label{alg:rand-map}
  {\small{
  \begin{algorithmic}[1]
        \Require {$x\in\mathbb{D}_\bot$, parameter $\theta\in\mathbb{N}^+$, $\theta\le |\mathbb{D}_{\bot}|$, $\epsilon_{prt}>0$}
        \Ensure  {$\mathcal{P}_{\hat{m}}\in\{\mathcal{P}_1,...,\mathcal{P}_k\}$}
        \State {Partition $\mathbb{D}_{\bot}$ into $k$ partitions: $\mathcal{P}_1, \mathcal{P}_2, ..., \mathcal{P}_k$, and $x$ is located in the partition $\mathcal{P}_{m}$.}
        \For {$i\in [k]$}
        \State {$p_{x,i}=Pr[\hat m=i]=\frac{e^{-|m-i|\cdot\epsilon_{prt}/2}}{\sum_{j\in[k]}e^{-|m-j|\cdot\epsilon_{prt}/2}}$}
        \EndFor
        \State {Sample $\hat m \sim p_x=\{p_{x,1},p_{x,2},...,p_{x,k}\}$}\\
        \Return {$\mathcal{P}_{\hat{m}}$}
  \end{algorithmic}}}
\end{algorithm}

\begin{definition}\label{def:partition-dlocal-differential-privacy} (\emph{Partition-dLDP}). 
  An algorithm $\mathcal{M}$ satisfies ($\epsilon_{prt}$, $\epsilon_{ner}$)-partition-dLDP, if and only if for any input $x, x' \in \mathbb{D}_{\bot}$ such that $|x-x'|\le t$, $\mathbb{D}_{\bot}$ is equally divided into several partitions of length $\theta$, and any output $y \in Range(\mathcal{M})$, we have 
  \[\setlength\abovedisplayskip{0.5ex}
     \setlength\belowdisplayskip{0.5ex}
  \Pr \left[\mathcal{M}(x) = y \right] \leq e^{\lceil\frac{t}{\theta}\rceil\epsilon_{prt}+\theta\epsilon_{ner}}\cdot \Pr \left[\mathcal{M}(x') = y \right].\]
\end{definition}

Compared with the existing dLDP definition, partition-dLDP can allocate privacy budget more finely based on both $l_1$ distance and partition distance.
We find that the existing dLDP definition is just a special case when $\epsilon_{prt}$ and $\epsilon_{ner}$ satisfy the following relationship
\begin{equation}
  \label{equ:fit_base}
     \begin{cases}
      \epsilon_{prt} = \theta\epsilon_{ner}\\
      \epsilon_{ner}=\frac{\epsilon}{1+\theta/|\mathbb{D}_{\bot}|}\\
     \end{cases}
\end{equation}
By adjusting $\epsilon_{prt}$ and $\epsilon_{ner}$, we can get more kinds of probability distributions of desensitized values than the existing dLDP mechanisms.
Therefore, it can meet a wider range of privacy policies and seek for achieving better utility.
\ifshaphered
\emph{(Reviewer2:D2)}
\else
\fi
\revision{
Note that the partition-dLDP in Definition \ref{def:partition-dlocal-differential-privacy} assumes that the feature can be evenly divided according to its semantic information.
Some features may cannot be evenly divided into intervals according to its semantics, i.e., the range of $[0,60]$ is unqualified, $[60,80]$ is qualified, and $[80,100]$ is excellent for \emph{score} feature.
It suffices to only switch the partition strategy from evenly to unevenly, while the remaining processing of the proposed mechanisms in the paper can then be applied without any change.
However, it requires further generalizing the privacy budget assignment strategy to account for the extended uneven partition, which we leave for future work.
}

In order to verify our theoretical analysis, we design a mechanism that satisfies partition-dLDP, called Adj-map. 
Instead of randomly mapping sensitive values based on distance in the whole domain $\mathbb{D}_{\bot}$, Adj-map first randomly selects a partition using Global-map satisfying $\epsilon_{prt}$-dLDP as the output domain.
The details are described in Algorithm \ref{alg:rand-map}.
Then the sensitive value is randomly mapped to a value in this partition using Global-map satisfying $\epsilon_{ner}$-dLDP.
We then prove that Adj-map satisfies partition-dLDP in Theorem \ref{the:adj-map}.

\begin{theorem}
  \label{the:adj-map}
  Adj-map satisfies partition-dLDP.
\end{theorem}
\vspace{-1em}
\begin{proof}
  The proof is presented in Appendix \ref{app:adj_map}.
\end{proof}
In Figure \ref{fig:distribution}, we show the output probability distribution of a sensitive value desensitized by Adj-map and Global-map when providing the approximately same privacy protection for any pair of values.
The distribution of output probabilities between partitions and within partitions can be changed by adjusting the ratio of $\epsilon_{prt}$ and $\epsilon_{ner}$.
When $\epsilon_{prt}=\alpha\theta\epsilon_{ner}$, $\epsilon_{ner}=\frac{\epsilon}{\alpha+\theta/|\mathbb{D}_{\bot}|}$.
Increasing $\alpha$ can make the desensitized value remains in the original partition with a greater probability.
Meanwhile, the distribution of probabilities in each partition is more uniform. 
Thus, the probability of reverse order of value pair with large distance is reduced and the indistinguishability of the closed values is increased.
Figure \ref{fig:distribution} also confirms that when Equation \ref{equ:fit_base} is satisfied, the probability distribution of output values of Adj-map and Global-map is almost entirely fitted.
The Global-map satisfying existing dLDP is just a special case when $\epsilon_{prt}$ and $\epsilon_{ner}$ in Adj-map meet Equation \ref{equ:fit_base}.
Then we theoretically show the order-preserving probability of Adj-map in Theorem \ref{the:randmaputi}.

\begin{theorem}
  \label{the:randmaputi}
  Adj-map is a probabilistic order-preserving desensitization algorithm with $\gamma(t)\ge1-q^{T}(\frac{(1-q^2)\cdot T+1}{(1+q-q^{T+1}-q^{k-T})(1+q)}-\frac{(1-q)^2(T+1)}{2(1+q)^2})$, where $q=e^{-\epsilon_{prt}/2}, T=\lfloor \frac{t}{\theta}\rfloor$.
\end{theorem}

\begin{proof}
  The proof is presented in Appendix \ref{app:uti-adj-map}.
\end{proof}

\noindent \textbf{Local-Map:}
The sensitive value can be mapped to the partition where it is located with the greater probability when $\epsilon_{prt}$ increases.
When $\epsilon_{prt}=\infty$, the output domain of the desensitized value is reduced from $\mathbb{D}_{\bot}$ to the partition where the sensitive value is located.
Therefore, the sensitive value pairs in different partitions always remain the order after desensitization.
Specifically, each sensitive value is desensitized by Global-map satisfying $\epsilon_{ner}$-dLDP in the partition that it is located.
We show the privacy guarantee of Local-map in Theorem \ref{the:direct-map}.

\begin{theorem}
  \label{the:direct-map}
  The privacy guarantee provided by Local-map satisfies the following:
  (1) For any pair of values $x$, $x'$ are in different partitions, $x$, $x'$ can be distinguished.
  \[\setlength\abovedisplayskip{0.5ex}
     \setlength\belowdisplayskip{0.5ex}
  \exists o\in\mathbb{D}_{\bot}, Pr[O=o|x]\neq 0, Pr[O=o|x']=0.\]
    
  (2) For any pair of values $x$, $x'$ are in the same partition, $x$, $x'$ satisfies $\epsilon_{ner}$-dLDP, where $|x'-x|\le t\le \theta, \epsilon_{ner}>0$.
  \[\setlength\abovedisplayskip{0.5ex}
     \setlength\belowdisplayskip{0.5ex}
  \forall o\in\mathbb{D}_{\bot}, Pr[O=o|x]\le e^{t\cdot\epsilon_{ner}}\cdot Pr[O=o|x'].\]
\end{theorem}

The order-preserving probability $\gamma(t)=1$ for value pairs located in different partitions, while $\gamma(t)$ is the same as that of Global-map stated in Theorem \ref{the:baseuti} when value pairs are in the same partition.
When providing the same privacy guarantee for value pairs in the same partition, Local-map can preserve more order information than Adj-map and Global-map.
The finer the granularity of the partition, the more order information the desensitized values retain.

\subsection{Improve Efficiency with (Bounded) Discrete Laplace}
We have utilized a variant of the exponential mechanism to construct all our aforementioned order-preserving desensitization algorithms. 
Such constructions have two limitations.
First, the computational complexity of sampling in an exponential mechanism is proportional to the size of the output domain. 
If the output domain is large, the sampling procedure becomes inefficient. 
Second, the exponential mechanisms require bounded input/output domains, and the bounds of the domain need to be known in advance.
Although all values of a feature are held by one party, and the value domain is bounded and known in OpBoost, the proposed algorithms are difficult to desensitize distributed values in other complex scenarios.
On the other hand, if the input domain is unknown in advance, one needs to collect all input values \emph{before} determining desensitization parameters, introducing additional deployment requirements. 
Since releasing the output domain may violate privacy, other differentially private mechanisms would be introduced to determine the output domain (e.g., \cite{wilson2020differentially}) in a privacy-preserving manner. 

In this section, we focus on the situations where the input domain is unknown in advance, or the output domain is large. 
We introduce (bounded) discrete Laplace mechanism in our order-preserving desensitization algorithms as an alternative to the exponential mechanism. 
The discrete Laplace mechanism supports unbounded input/output sampling so that the input/output domain can be treated as infinity. 
Instead of leveraging the Inverse Cumulative Distribution Function (Inverse CDF) sampling method as in the exponential mechanism, the sampling procedure for discrete Laplace mechanism is independent of the domain size and more efficient \cite{canonne2020discrete}. 
It remains to show that replacing the exponential mechanism with the discrete Laplace mechanism can still provide the privacy guarantee. 
We first consider the infinite input/output domain setting. 
We show the definition of discrete Laplace distribution in Definition \ref{def:discrete-laplace}.
\begin{definition}\label{def:discrete-laplace} (\emph{Discrete Laplace Distribution}). 
  The discrete Laplace distribution with scale parameter $\lambda$ is denoted $Lap_\mathbb{Z}(\lambda)$, where $\lambda=1/\epsilon$.
  Its probability distribution can be defined as
  \[\setlength\abovedisplayskip{0.5ex}
     \setlength\belowdisplayskip{0.5ex}
  \forall z\in\mathbb{Z},\ Pr[Z=z]=\frac{e^{1/\lambda}-1}{e^{1/\lambda}+1}\cdot e^{-|z|/\lambda}.\]
\end{definition}
Then we show in Theorem \ref{the:discrete-laplace} that adding noise following discrete Laplace distribution satisfies dLDP.

\begin{theorem}\label{the:discrete-laplace}
  Any pair of values $x_1$ and $x_2$ with $|x_1-x_2|\le t$, satisfies $\epsilon$-dLDP after adding the noise sampling from discrete Laplace distribution $Lap_{\mathbb{Z}}(\frac{1}{\epsilon})$.
\end{theorem}

\begin{proof}
  The probability ratio of $x_1$ and $x_2$ being randomized to the same output value $o$ is 
  \setlength\abovedisplayskip{0.5ex}
   \setlength\belowdisplayskip{0.5ex}
  \begin{align*}
    \frac{Pr[x_1+N_1=o]}{Pr[x_2+N_2=o]}&=\frac{Pr[N_1=o-x_1]}{Pr[N_2=o-x_2]}=\frac{e^{-|o-x_1|/\lambda}}{e^{-|o-x_2|/\lambda}}\le e^{t\cdot\epsilon}.
  \end{align*}
\end{proof}
\vspace{-1em}

Theorem \ref{the:discrete-laplace} shows that, in the infinite input/output domain setting, Global-map can still satisfy $\epsilon$-dLDP when sampling the desensitized output values from the discrete Laplace distribution. 
Randomizing the partition in Adj-map using the discrete Laplace distribution also does not change the privacy guarantee of partition algorithms proved in Subsection \ref{sec:algorithms}.

We next consider the large (but finite) input/output domain setting, where the discrete Laplace mechanism can also be utilized as an alternative to the exponential mechanism. 
However, due to its unbounded output domain, we need to consider the case where the discrete Laplace mechanism samples a value outside the desired output domain. 
We solve this by resampling, and we call the mechanism as \emph{bounded} discrete Laplace mechanism. We first give the probability distribution of the bounded discrete Laplace distribution obtained by resampling in Lemma \ref{lem:bounded-laplace}.
\begin{lemma}\label{lem:bounded-laplace}
  Given the sampling range $[l,u]$, the probability distribution of the bounded discrete Laplace distribution is
  \[\setlength\abovedisplayskip{0.5ex}
   \setlength\belowdisplayskip{0.5ex}
  \forall z\in [l,u],\ Pr[Z=z]=\tau\cdot\frac{e^{1/\lambda}-1}{e^{1/\lambda}+1}\cdot e^{-|z|/\lambda},\]
  where $\lambda=1/\epsilon$, and
  \[\tau=
  \begin{cases}
  \frac{2}{e^{u/\lambda}(1-e^{-(u-l+1)/\lambda})}, &l<u<0,\\
  \frac{2}{1-e^{-(-l+1)/\lambda}-e^{-(u+1)/\lambda}+e^{-1/\lambda}}, &l<0<u,\\
  \frac{2}{e^{-l/\lambda}(1-e^{-(u-l+1)/\lambda})}, &0<l<u.\\
  \end{cases}
  \]
\end{lemma}
\begin{proof}
  The proof is deferred in Appendix \ref{app:distribution_bdldp}.
\end{proof}

Next, we analyze the privacy guarantee of bounded discrete Laplace mechanism in Theorem \ref{the:pribounded-laplace}. 

\begin{theorem}\label{the:pribounded-laplace}
  For any output range $[l,u]$, any pair of input values $x_1$ and $x_2$ with $|x_1-x_2|\le t$ satisfies $2\epsilon$-dLDP after adding the noise sampling from bounded discrete Laplace distribution $Lap_{\mathbb{Z}}(\frac{1}{\epsilon})$.
\end{theorem}

\begin{proof}
  The proof is deferred in Appendix \ref{app:pribounded-laplace}.
\end{proof}
Theorem \ref{the:pribounded-laplace} shows that for Global-map, sampling the output value from the bounded discrete Laplace distribution $Lap_{\mathbb{Z}}(\frac{2}{\epsilon})$ can provide $\epsilon$-dLDP privacy guarantee. 
Also, for Adj-map, and Local-map, replacing the exponential distribution $Exp(\frac{1}{\epsilon})$ with bounded discrete Laplace distribution $Lap_{\mathbb{Z}}(\frac{2}{\epsilon})$ when mapping the partition and sampling output values in the partition does not affect the privacy guarantee, as shown in Corollary \ref{cor:bound-laplace-oth}. 

\begin{corollary}
  \label{cor:bound-laplace-oth}
  Sampling from the bounded discrete Laplace distribution $Lap_{\mathbb{Z}}(\frac{2}{\epsilon})$ instead of exponential mechanism has no effect on privacy guarantee provided by Adj-map and Local-map.
\end{corollary}

One may also wonder if additional privacy budgets should be consumed when resampling occurs. Observe that if the output distribution satisfies the privacy guarantee, only the time consumption for sampling reflects if resampling happens. 
In the situation where the adversary has a very strong capability of carrying out the side-channel attack, we also recommend considering extra privacy budget consumptions.
In addition, the probability of resampling increases as the input/output domain decreases.
Therefore, we still recommend using the exponential mechanism for tasks with a small input/output domain.

\section{Theoretical Analysis}
\subsection{Utility Analysis of OpBoost}
\label{sub:error_split}
\ifshaphered
\emph{(Reviewer1:D3)}
\else
\fi
\revision{
Here, we provide theoretical evidence for the utility of OpBoost.
\begin{theorem}
\label{the:max_gain}
The probability that no desensitized values of a feature crosses any potential split point $x^{\dagger}\in[L,R]$ after sorting is at least $\beta$.
where
  {\setlength\abovedisplayskip{0.5ex}
  \setlength\belowdisplayskip{0.5ex}
  \begin{align*}
    \beta&=\sum_{k\in\mathcal{I}_l}\sum_{x^{\dagger}\in[L,R]} (\mathcal{M}(k,x^{\dagger})\cdot\prod_{j\neq k,j\in \mathcal{I}_l}\sum_{x_l\in{[L,x^{\dagger}]}}\mathcal{M}(j,x_l)\\
    &\cdot\prod_{j\in\mathcal{I}_r}\sum_{x_r\in[x^{\dagger}+1,R]}\mathcal{M}(j,x_r)).
  \end{align*}}
Here, $\mathcal{I}_l$ (resp. $\mathcal{I}_r$) is the set of feature values on the left (resp. right) of the split point.
The function $\mathcal{M}(x,y)$ outputs the probability of desensitizing the value $x$ to $y$ with a desensitization algorithm.
$\beta$ contains the probability that any value $k$ in $\mathcal{I}_l$ is desensitized as the left maximum value $x^{\dagger}$ and the other values $x_l\neq k$ in $\mathcal{I}_l$ are all lower than $x^{\dagger}$, the values $x_r$ in $\mathcal{I}_r$ are all greater than $x^{\dagger}$ after desensitizing.
\end{theorem}
}
\begin{figure}[t]
  \centering
  \includegraphics[width=0.48\textwidth]{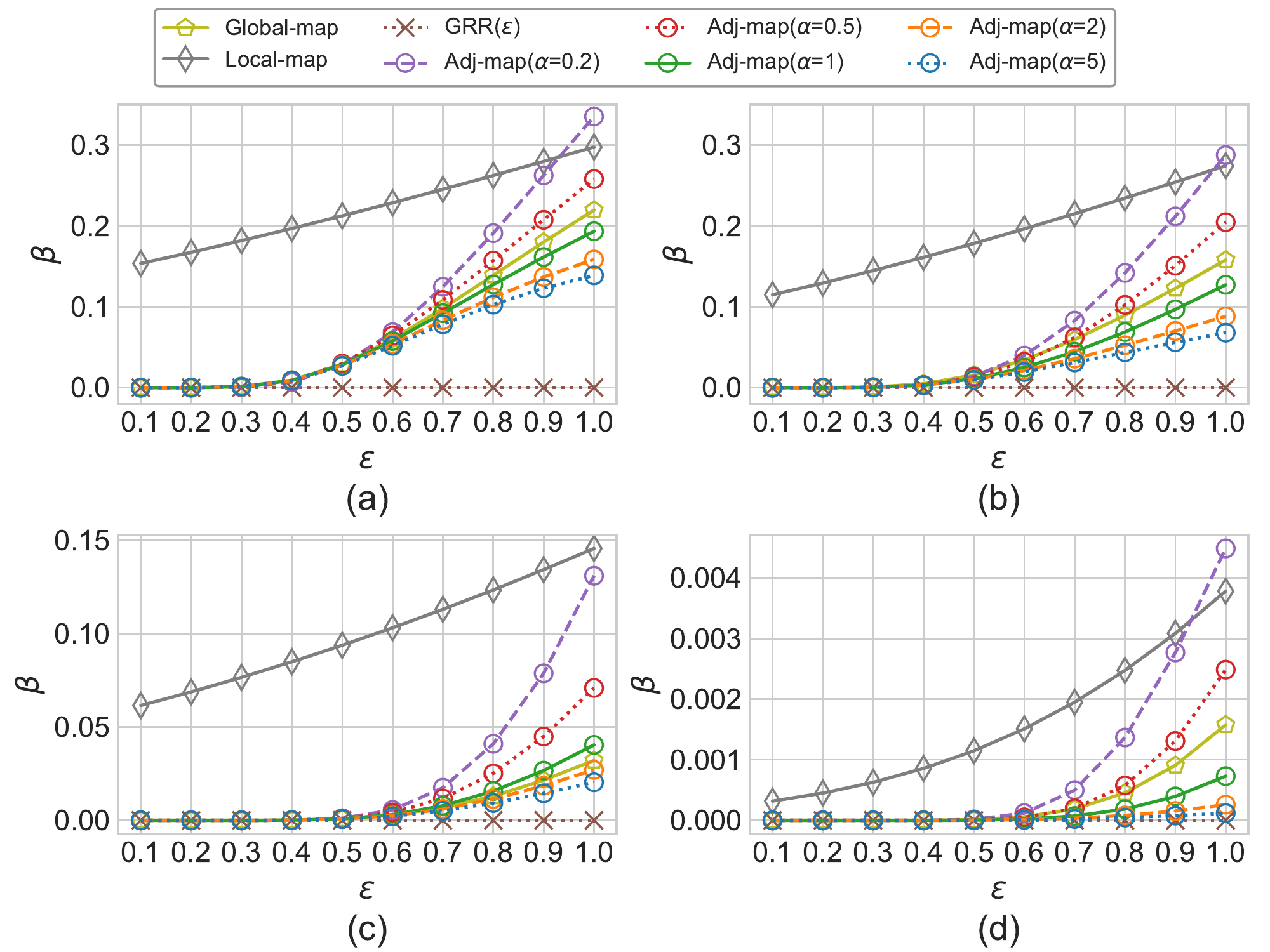}
  \setlength{\belowcaptionskip}{0ex}
  \caption{Comparison of $\beta$ calculated by the algorithms on a uniformly and a normal distributed datasets with $100$ values, where $\theta=4$, $|\mathbb{D}_\bot|=100$.
  (a)(b) are on uniform datasets, and (c)(d) are on normal datasets.
  (a)(c) and (b)(d) are the split point at the $25\%$ and $50\%$ quantiles, respectively.}
  \label{fig:split_prob}
\end{figure}

\revision{
In Theorem \ref{the:max_gain}, the distribution and density of the feature values and the location of the split point are intertwined to influence the value of $\beta$.
To provide theoretical evidence, we have to analyze $\beta$ under the condition that all these factors are controllable.
We choose two commonly used distributions to calculate $\beta$ when the split point is located at $25\%$ and $50\%$ quantiles.
Figure \ref{fig:split_prob} shows the comparison of $\beta$ between the proposed order-preserving desensitization algorithms and the LDP mechanism.
The denser or more centralized the distribution of the datasets, the higher the number of disordered value pairs after desensitization, thus reducing the utility of the desensitized values. 
Besides, we observe that $\beta$ decreases as the split point approaches the median of the feature values.
When the split point at $25\%$ of the feature values, $\beta$ can reach $10$ times that of the split point at $50\%$ of the feature values for the normal dataset.
Although many factors have a great impact on $\beta$, the proposed algorithms are always better than the LDP method.
\begin{corollary}
\label{cor:maximum_gain}
  Let $G$ be the maximum gain when splitting a feature, and $\hat G$ be the maximum gain after desensitizing feature's values. 
  We have
  \[\setlength\abovedisplayskip{0.5ex}
   \setlength\belowdisplayskip{0.5ex}
   Pr[\hat G \ge G]\ge \beta.\]
\end{corollary}
When the optimal split point with maximum gain $G$ is chosen in Theorem \ref{the:max_gain}, we can give a theoretical lower bound of the probability that the maximum gain $\hat G$ is not less than $G$ after desensitizing as shown in Corollary \ref{cor:maximum_gain}.
Given the fact that the exhaustive split point searching will traverse all potential split points, it is guaranteed that the current split point obtained with desensitizing always has the largest maximum gain, because otherwise, we can find another split point with an even greater maximum gain during the exhaustive search. 
Intuitively, we provide a probability lower bound that the accuracy of the model trained based on desensitized features is not lower than that of the plain model.
}

\subsection{Efficiency Analysis of OpBoost}
\ifshaphered
\emph{(Reviewer1:D2.c,Reviewer3:D2)}
\else
\fi
\revision{
We analyze the computational and communicational complexities of OpBoost to provide the overhead in theory.
Without loss of generality, we consider the setting consisting of m-1 users in party B and one user in Party A, where each user holds at most $r$ features with $N$ total training samples.
}
\revision{
\begin{theorem}
  \label{the:computaition_overhead}
  The total computation overhead for each user in party B is $O(Nr)$, and $O(NT+(2^L-1)T)$ for the user in Party A, where $T$ is the number of trees, $L$ is the number of tree layers.
\end{theorem}
\begin{proof}
The computation overhead of Party B comes from desensitizing all features' values.
The computational complexity is $O(1)$ by sampling a desensitized value from discrete Laplace distribution.
Thus, the total overhead is at most $O(Nr)$ for each user in party B.
For Party A, the complexity of model training is the same as that of non-private training.
The additional computational complexity comes from replacing all tree nodes with desensitized values after training, which is $O((2^L-1)T)$.
Therefore, the total computation overhead of party A is $O(NT + (2^L-1)T)$.
\end{proof}
}
\revision{
To show the efficiency of the proposed desensitization algorithms, we compare them with an Order-Preserving Encryption (OPE) scheme.
As shown in Appendix \ref{app:ope_compare}, our proposed algorithms are $40$-$200$ times faster than OPE desensitizing the same value set.
The training process is the same as that of the non-private tree boosting algorithm \cite{friedman2001greedy, chen2016xgboost}.
}
\revision{
\begin{theorem}
  \label{the:comunication_overhead}
  The communication channel of OpBoost needs to transmit $O(mrN\log N+(2^L-1)T(\log N+\log |\mathbb{D}_\bot|))$ bits in total, where $T$ is the number of trees, $L$ is the number of tree layers.
\end{theorem}
\begin{proof}
There are $O(mrN\log N)$ bits desensitized values sent from Party B to Party A before the training.
No information exchange is required during the training.
Party A sends at most $O((2^L-1)T\log N)$ bits to Party B and receives $(2^L-1)T\log |\mathbb{D}_\bot|$ to replace the tree nodes at the end of the training.
\end{proof}
}
\section{Experimental Evaluation}
In this section, we empirically evaluate the performance of OpBoost.
In experiments, we measure 
(1) the order preservation of a numerical dataset after desensitizing by the proposed order-preserving desensitization algorithms, i.e., how much relative order information between values is preserved, 
the impact of the ratio of privacy budget parameters on the order preservation of desensitized results.
\revision{
(2) the accuracy of the proposed order-preserving desensitization algorithms for order-dependent statistical applications,
i.e., how do they compare with the state-of-the-art LDP-based range query method in accuracy.
}
(3) the performance of decision tree models trained by OpBoost, i.e., how much does desensitization of features' values affect the accuracy of trained models;
\revision{What is the computation and communication overhead of OpBoost during the training.}
Towards these goals, we run the GBDT and XGBoost training algorithm in OpBoost for both classification and regression tasks.
Following the general settings in tree boosting algorithms, we let $T=80$ (number of trees), $\eta=0.1$ (learning rate), and $L=3$ (tree layers) throughout the experiment.

\subsection{Setup}
\noindent\textbf{Datasets.}
\ifshaphered
\emph{(Reviewer1:D2.b)}
\else
\fi
\revision{
  We conduct range query evaluation using a real-world dataset \emph{Salaries} \cite{salaries-web} and a synthetic dataset.
  \emph{Salaries} \cite{salaries-web} is also considered in our competitor AHEAD \cite{du2021ahead}, which contains $148,654$ records.
We follow the same practice with AHEAD to map values into the range of $[1, 1024]$ for the fairness.
}

We conduct GBDT and XGBoost on four public datasets and a large-scale industrial dataset, which are listed in Table \ref{tab:dataset}.
For datasets with all samples in one file, we randomly select $80\%$ and $20\%$ samples for training and testing.
\begin{table}[h]
  \small
  \centering
  \setlength{\parskip}{-0.35cm}
  \begin{tabular}{|c|c|c|c|}
    \hline
      Dataset&\#Instances&\#Features(N/E)&Tasks\\
    \hline
      \emph{Adult}\cite{adult-web}&32.6k&4/10&2-Cls.\\
    \hline
      \emph{Pen-digits}\cite{pendigits-web}&11k&16/0&M-Cls.\\
    \hline
      \emph{Powerplant}\cite{plant-web}&9.5k&4/0&Reg.\\
    \hline
      \emph{CASP}\cite{casp-web}&45.7k&9/0&Reg.\\
    \hline
    \emph{Industrial}&293.6k&38/262&Reg.\\
    \hline
  \end{tabular}
  \setlength{\abovecaptionskip}{0ex} 
  \setlength{\belowcaptionskip}{0ex}
  \caption{Description of datasets, where N and E refer to Numeric and Enumeration, respectively.}
  \label{tab:dataset}
\end{table}

\noindent\textbf{Metrics.}
The metrics we use are as follows.

\textbf{\it Weighted-Kendall.}
We use weighted-Kendall as the metric to evaluate the order preservation of the desensitized values.
The formal definition of weighted-Kendall is given by Vigna \cite{vigna2015weighted}.
\begin{definition}
\label{def:weighted-kendall} (\emph{Weighted-Kendall}).
For two real-valued vectors $r$ and $s$, weighted-kendall $\tau_w(r,s)$ is defined as
\begin{align*}
  \frac{\langle r,s\rangle_w}{\sqrt{\langle r,r\rangle_w}\sqrt{\langle s,s\rangle_w}}, where
  \ sgn(x):=
  \begin{cases}
    1\ \ &if\ x>0;\\
    0\ \ &if\ x=0;\\
    -1\ \ &if\ x<0.\\
  \end{cases}
\end{align*}
$\langle r,s\rangle_w=\sum_{i<j}sgn(r_i-r_j)sgn(s_i-s_j)w(i,j)$,
\end{definition}
\noindent$\tau_w(r,s)=1$ means that the vector is strictly order-preserving, while $\tau_w(r,s)=-1$ means that the vector is completely reversed.
In experiments, we consider the vector $r$ as the raw sensitive values and the vector $s$ as the desensitized values.
We define the weight function $w(i,j)$ as the distance of ordinal numbers between $r_i$ and $r_j$ after sorting the vector $r$.

\textbf{\it Accuracy.}
We evaluate the accuracy of the prediction results of the model for \emph{classification} as following
\[\setlength\abovedisplayskip{0.5ex}
\setlength\belowdisplayskip{0.5ex}
Accuracy=\frac{\#Correct\ Predictions}{\#Correct\ Predictions+\#Wrong\ Predictions}\%.\]

\textbf{\it Mean Square Error.}
We use MSE as the metric to evaluate the statistical results of range query and the prediction results of the model for \emph{regression}.
For $n$ testing samples, we calculate the squared difference between each prediction result $\tilde y_i$ from OpBoost and the corresponding results $y_i$ from model without privacy protection.
\[\setlength\abovedisplayskip{0.5ex}
\setlength\belowdisplayskip{0.5ex}
MSE=\frac{1}{n}\sum_{i\in [n]}(\tilde y_i-y_i)^2.\]
\revision{\noindent The results of range query tasks are averaged with 100 repeats, and the results of other tasks are all averaged with 10 repeats.}

\noindent\textbf{Environment.}
We do our experiments on a single Intel Core i9-9900K with \revision{$3.6$}GHz and $4 \times 32$GB RAM, running Ubuntu 20.04.2 LTS.
\revision{We execute our protocols on two progresses, one for Party A and the other for Party B. 
The network connection between the two progresses are built via the local network. We emulate the LAN network setting with latency $0.02$ms and bandwidth $10$Gbps using the Linux \textsf{tc} command. 
We use an asynchronous event-driven network application framework \textsf{Netty} to maintain the network connection, and use the well-known tool \textsf{Protocol Buffers} for data serialization and deserialization.
Our schemes are implemented in Java with the multi-thread support by the Fork-Join concurrency technique. While the source code is based on Java 8, we run experiments on Java HotSpot(TM) with higher version 17.0.2 to have better performance reports.}
\begin{figure}[t]
  \centering
  \includegraphics[width=0.48\textwidth]{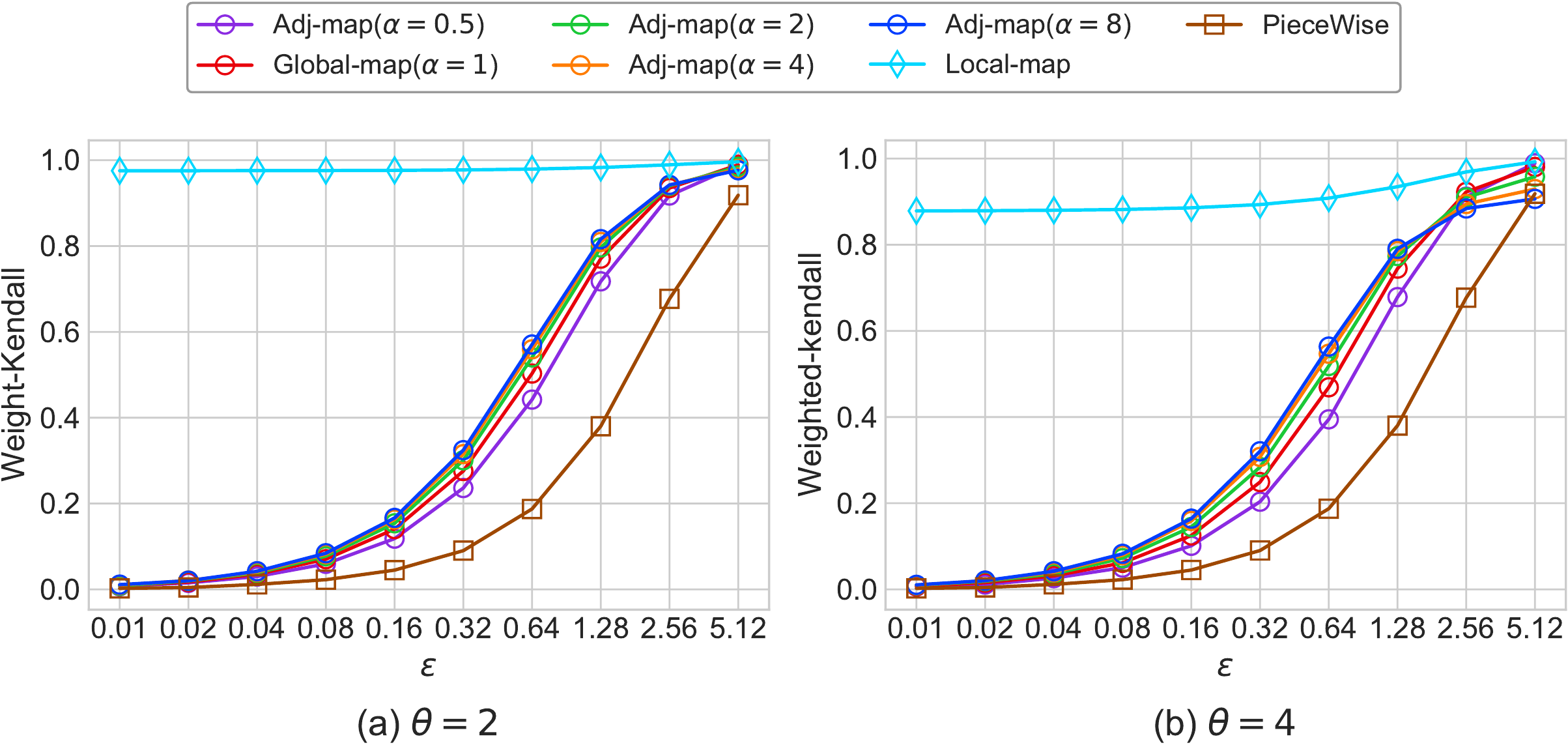}
  \setlength{\belowcaptionskip}{0ex}
  \caption{Weighted-Kendall on \emph{age} of Adult Dataset Containing $32.6k$ Age Values.}
  \label{fig:adult_kendall}
\end{figure}

\begin{figure}[t]
  \centering
  \includegraphics[width=0.48\textwidth]{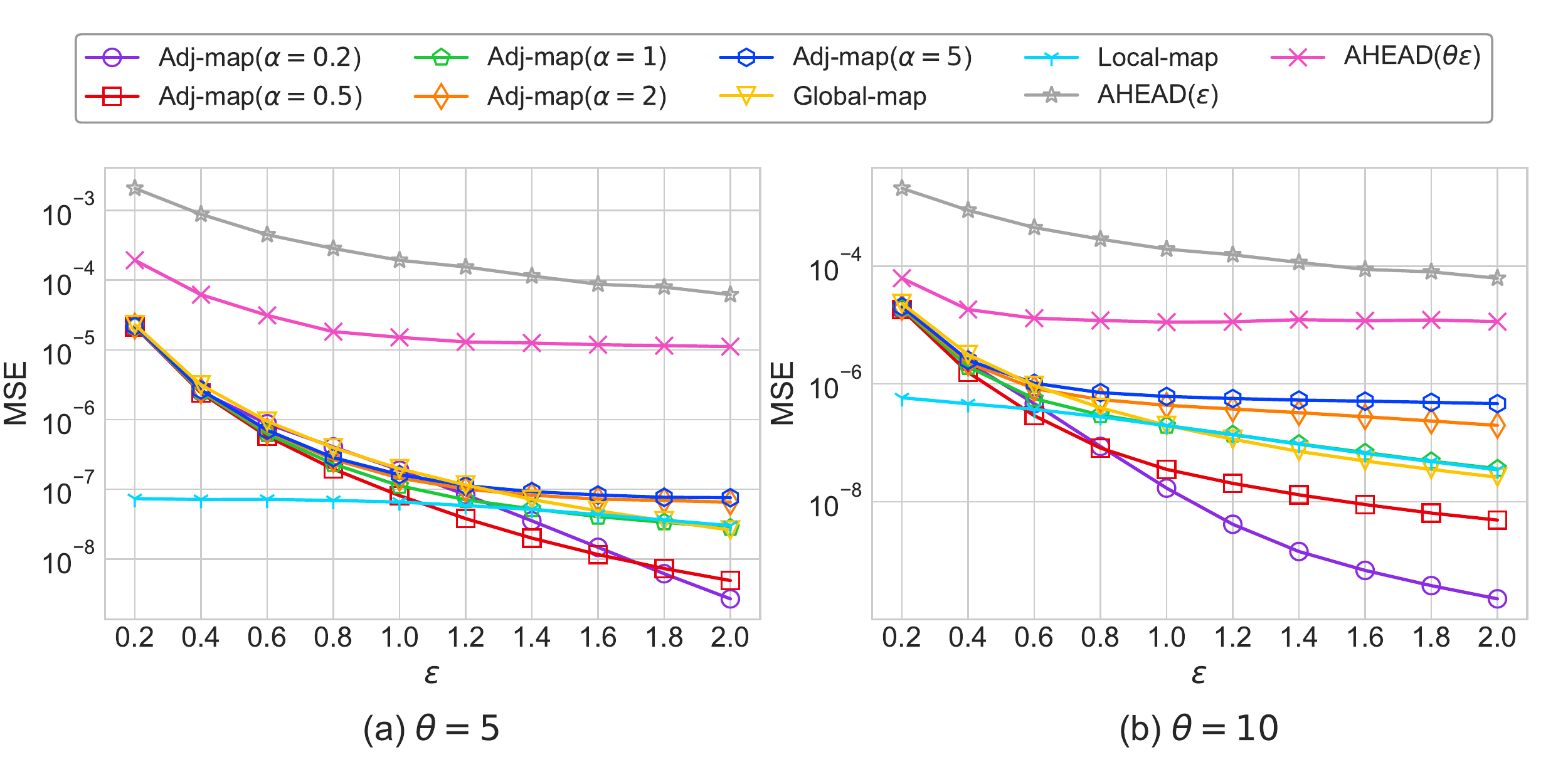}
  \setlength{\belowcaptionskip}{0ex}
  \caption{The MSE of Range Query on Salaries Dataset.}
  \label{fig:rq_salaries}
\end{figure}

\subsection{Experimental Details}
\revision{\noindent\textbf{Range Query.}
Our algorithms are not optimized by any data structures, while our competitor AHEAD uses a complex tree structure to improve the utility.
We randomly generate $10k$ queries in the value range, and calculate the MSE of the frequency of real values and desensitized values falling within the query ranges.
}

\noindent\textbf{GBDT and XGBoost.}
All experiments are conducted with two parties' participation (any number of parties is supported), Party A cooperates with Party B.
We assign the enumeration features and label to Party A and distribute all the numerical features to Party B.
We preprocess all numerical features with different ranges into discrete values in the same range to facilitate setting privacy parameters.
Except for the range query tasks, all preprocessed values are mapped to the range of $[1,10]$.
We use the following libraries in our implementations.

\noindent\emph{GBDT.}
We use the codebase in \textsf{Smile}\cite{smile-web} to implement our federated GBDT training. 
\textsf{Smile} provides data abstraction `DataFrame' to allow us easily read data from files, add noises, encode/decode the randomized data for communication, and do the training. 

\noindent\emph{XGBoost}. 
We also do the implementation based on the well-known XGBoost library\cite{xgboost-web}.
We use the XGBoost JVM package to invoke XGBoost from Java. 
XGBoost model modification is implemented using \textsf{xgboost-predictor}\cite{xgboost-predictor-web}.
We adjust the code `Node' in `RegTreeImpl.java' to allow split condition replacement.


\vspace{-1em}
\subsection{Experimental Results}
\ifshaphered
\emph{(Reviewer1:D5)}
\else
\fi
\revision{
We evaluate the proposed algorithms from three perspectives, including 1) the order preservation capability under the given privacy requirements, 2) the utility for range query tasks, and 3) the utility for tree boosting algorithms. 
We not only compare all the proposed algorithms under the different parameter settings, but also compare them with the state-of-the-art LDP mechanism in the corresponding tasks.
Since the privacy definitions are different, it is impossible to compare the LDP algorithms and the proposed dLDP algorithms under the same privacy guarantee, which are the same case for the existing dLDP-related works \cite{wang2017local, he2014blowfish, alvim2018local}. 
We unify Global-map and Adj-map under the same $\epsilon$-dLDP privacy guarantee.
Specifically, we correspondingly set $\epsilon_{ner}$ to $\epsilon/(\alpha+\theta/|\mathbb{D}_\bot|)$ when we increase $\epsilon_{prt}$ to $\alpha\theta\epsilon_{ner}$ to ensure that Adj-map algorithms with different parameters are all provide $\epsilon$-dLDP privacy guarantee.
Moreover, Adj-map can be approximated as Global-map when $\alpha=1$.
Since Local-map does not provide privacy guarantee for values located in different partitions, it only provides $\epsilon$-dLDP privacy guarantee for values in the same partitions.
}
\begin{figure}[t]
  \centering
  \includegraphics[width=0.48\textwidth]{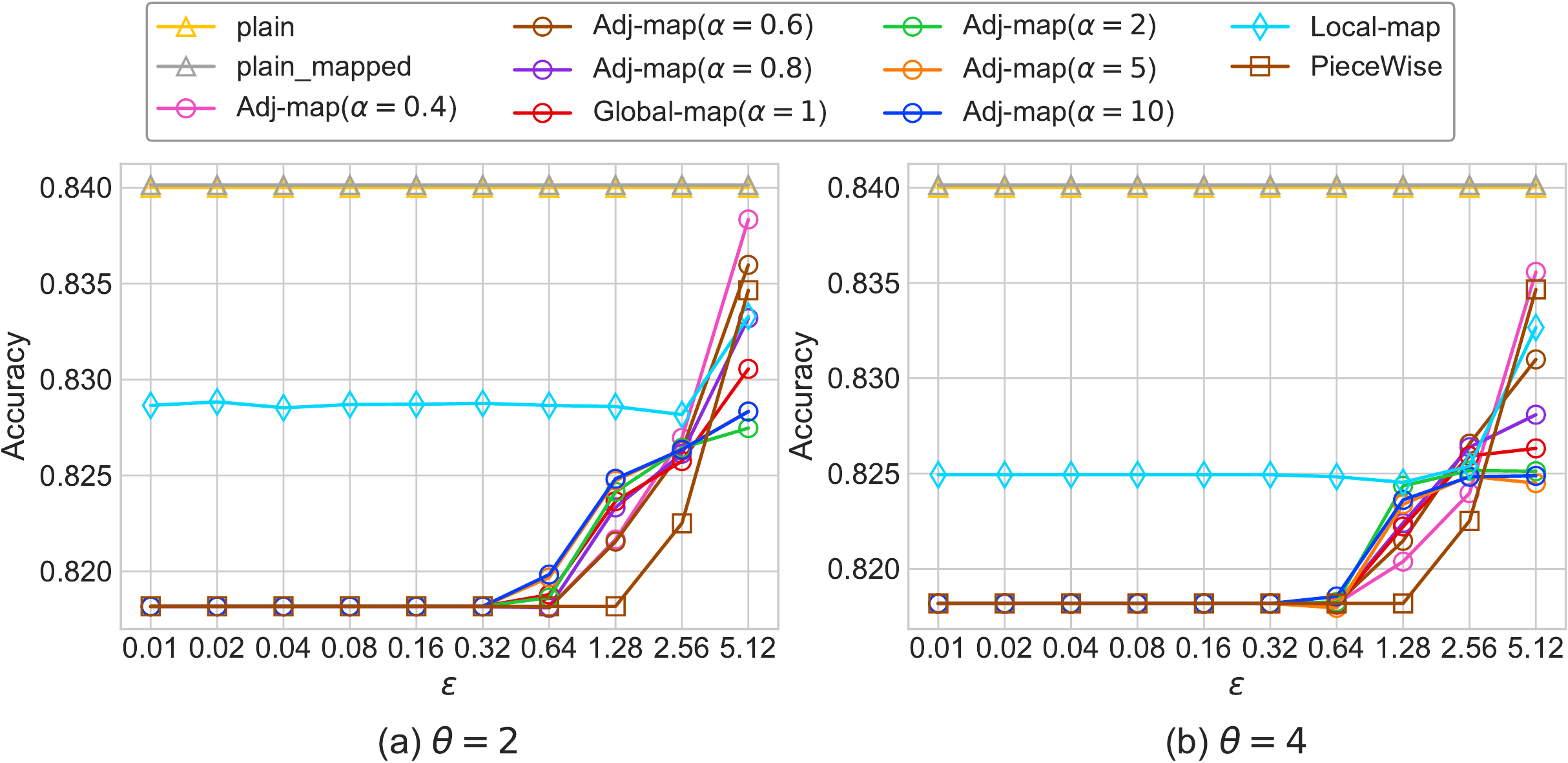}
  \setlength{\belowcaptionskip}{0ex}
  \caption{Prediction Precision of GBDT Models for Classification Trained on Adult Dataset.}
  \label{fig:adult_class}
\end{figure}

\begin{figure}[t]
  \centering
  \includegraphics[width=0.48\textwidth]{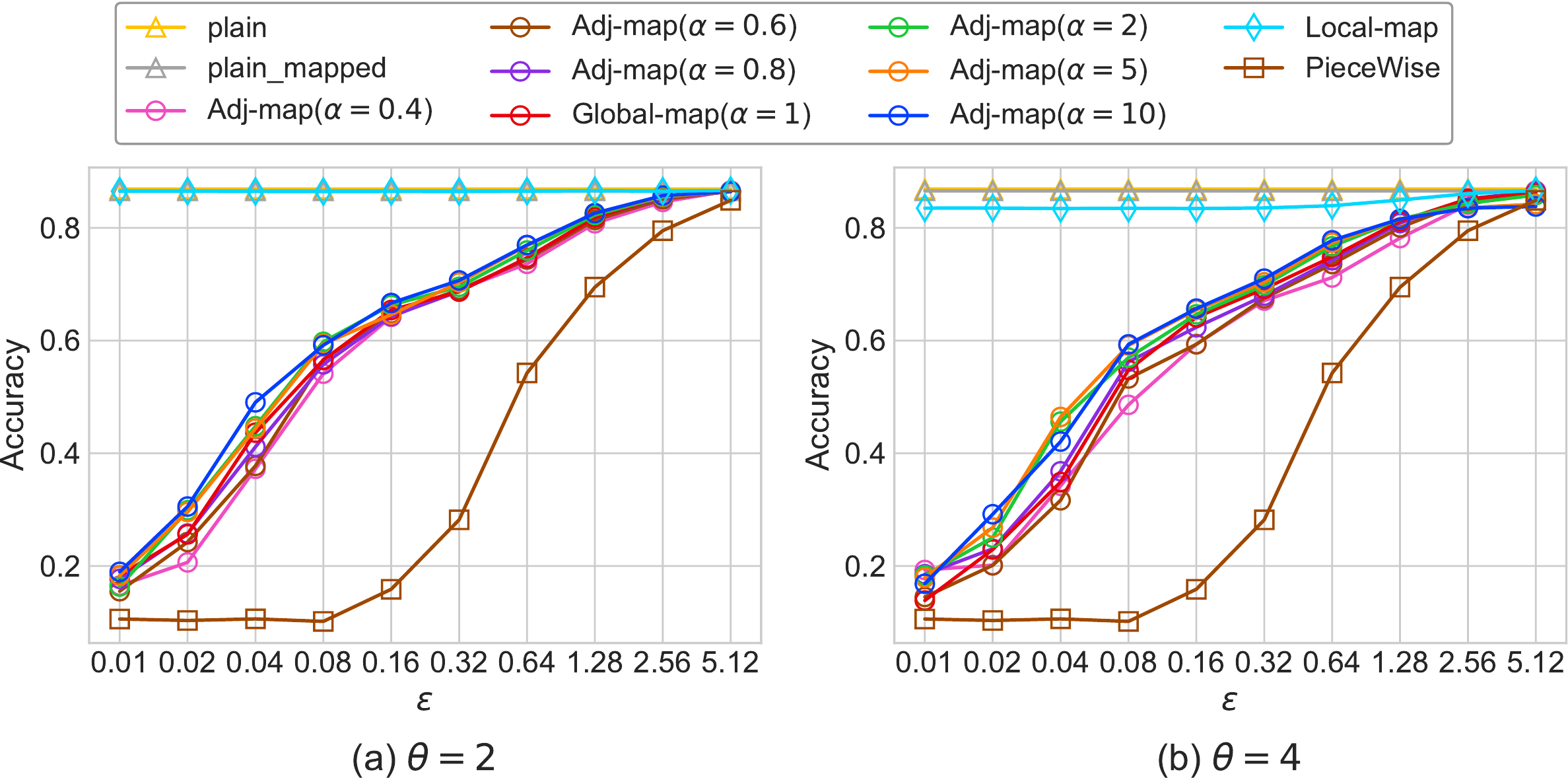}
  \setlength{\belowcaptionskip}{0ex}
  \caption{Prediction Precision of GBDT Models for Classification Trained on Pen-digits Dataset.}
  \label{fig:pendigits_class}
\end{figure}
\noindent\textbf{Order Preservation of Desensitized Values.}
We first evaluate the order preservation capability of the proposed order-preserving desensitization algorithms.
Although we evaluate the order preservation of a pair of values separated by different distances after desensitization in the previous section, it's not enough to reflect the overall order preservation of the entire desensitized value set.
Because the order-dependent applications generally rely on the order of the entire value set, and the farther away the values are out of order, the greater the impact on the results.
We selected values of a feature \emph{age} from the Adult dataset, and generate a uniformly distributed dataset with $10k$ values to evaluate the desensitized value sets by calculating the weighted-Kendall.
Note that the reason why a uniformly distributed value set is generated here is to facilitate the comparison of the performance of different algorithms on order preservation.
Since most of the value pairs in the unevenly distributed value set are not easily out of order after desensitization, it cannot reflect the improvement of the order preservation brought by optimizing desensitized value distribution. 

\begin{figure}[t]
  \centering
  \includegraphics[width=0.48\textwidth]{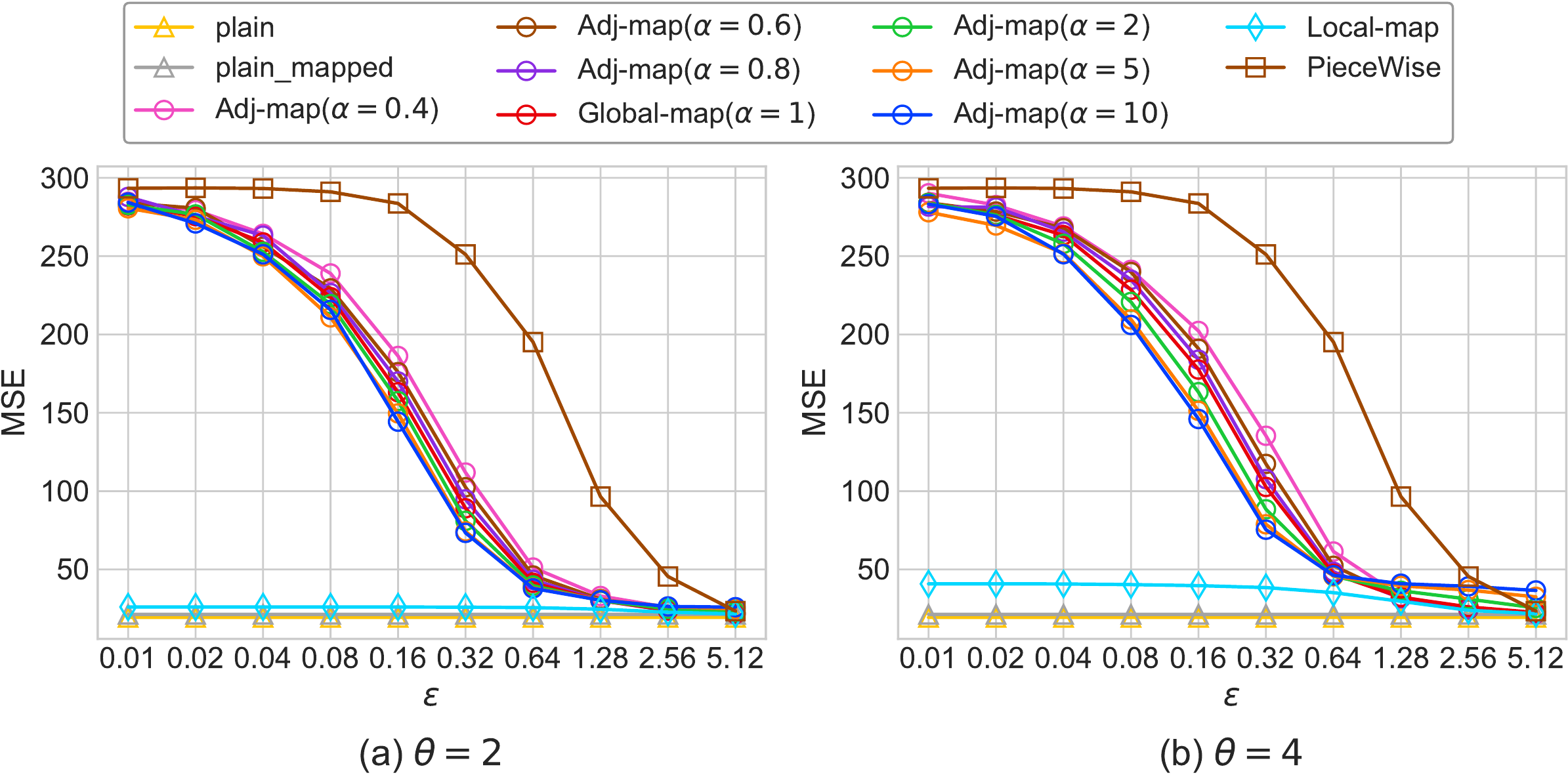}
  \setlength{\belowcaptionskip}{0ex}
  \caption{Prediction MSE of GBDT Models for Regression Trained on Powerplant Dataset.}
  \label{fig:plant_regress}
\end{figure}


\begin{table}[t]
  \small
  \centering
  \begin{tabular}{|c|c|c|c|c|c|}
      \hline
      \multirow{2}*{$\epsilon$}&\multirow{2}*{Method}&\multirow{2}*{MSE}&\multirow{2}*{Time(s)}&\multicolumn{2}{c|}{Communication(B)}\\
      \cline{5-6}
      & & & &Party A&Party B\\
      \hline
      \multirow{5}*{$0.64$}&{Local-map}&0.957$\times$&64.929&2961&202M\\
      \cline{2-2}\cline{3-6}
      &Adj-map($\alpha=0.5$)&5.341$\times$&99.183&2921&202M\\
      \cline{2-2}\cline{3-6}
      &Global-map($\alpha=1$)&5.273$\times$&98.448&2981&202M\\
      \cline{2-2}\cline{3-6}
      &Adj-map($\alpha=2$)&5.122$\times$&97.956&3113&202M\\
      \cline{2-2}\cline{3-6}
      &Piecewise&5.773$\times$&115.771&3097&202M\\
      \hline
      \multirow{5}*{$1.28$}&{Local-map}&0.951$\times$&64.207&2945&202M\\
      \cline{2-2}\cline{3-6}
      &Adj-map($\alpha=0.5$)&4.803$\times$&92.915&3137&202M\\
      \cline{2-2}\cline{3-6}
      &Global-map($\alpha=1$)&4.792$\times$&92.108&3413&202M\\
      \cline{2-2}\cline{3-6}
      &Adj-map($\alpha=2$)&4.295$\times$&91.495&3373&202M\\
      \cline{2-2}\cline{3-6}
      &Piecewise&6.150$\times$&110.488&3717&202M\\
      \hline
      \multirow{5}*{$2.56$}&{Local-map}&0.952$\times$&63.300&2901&202M\\
      \cline{2-2}\cline{3-6}
      &Adj-map($\alpha=0.5$)&1.063$\times$&86.247&3681&202M\\
      \cline{2-2}\cline{3-6}
      &Global-map($\alpha=1$)&1.002$\times$&86.278&3473&202M\\
      \cline{2-2}\cline{3-6}
      &Adj-map($\alpha=2$)&1.015$\times$&86.387&3525&202M\\
      \cline{2-2}\cline{3-6}
      &Piecewise&6.150$\times$&103.826&3053&202M\\
      \hline
  \end{tabular}
  \setlength{\abovecaptionskip}{0ex}
  \setlength{\belowcaptionskip}{0ex}
  \caption{Performance of XGBoost models trained by OpBoost on industrial dataset with $\theta=2$. We show the MSE ratio compared to the plain model.}
  \label{tab:umeng_opxgboost}
\end{table}

\revision{
Figure \ref{fig:adult_kendall} and Figure \ref{fig:syn10_kendall} (deferred to Appendix \ref{app:figure_kendall} due to limited space) show that the order preservation capability of Local-map is always better than that of other algorithms since it ensures that the orders of values in different partitions are strictly preserved.
In addition, we observe that when $\epsilon$ is small, i.e., $\epsilon<2.56$, Adj-map with a larger $\alpha$ can achieve better order preservation, and the opposite when $\epsilon>2.56$.
The reason is that the desensitized values fall far away from the raw values with a high probability when $\epsilon$ is relatively small.
A large $\alpha$ allocates more privacy budget to $\epsilon_{prt}$ so that the desensitized value is preserved in the partition where the raw value is located with a greater probability.
When $\epsilon$ is relatively large, the desensitized values are preserved in the raw values' partition with a high probability.
At this time, allocating more privacy budget to $\epsilon_{ner}$ to increase the order-preserving probability between values in the same partition is more helpful.
All the proposed algorithms are superior to the Piecewise mechanism \cite{wang2019collecting}, which is a widely accepted LDP mechanism for randomizing numerical values.
}

\ifshaphered
\emph{(Reviewer1:D2.b)}
\else
\fi
\revision{
\noindent\textbf{Utility Comparison for Range Query Tasks.}
Figure \ref{fig:rq_salaries} and Figure \ref{fig:rq_uniform} (deferred to Appendix \ref{app:figure_range_query} due to limited space) show the results of the proposed algorithms are applied to the range query tasks.
We compare with the state-of-the-art LDP-based algorithm AHEAD to demonstrate the utility improvement of the proposed algorithms on order-dependent statistical tasks by relaxing the privacy definition.
We use the same real-world dataset and data mapping method as AHEAD for the fairness of the comparison.
Consistent with the results of weighted-Kendall, $\alpha<1$ can get a smaller error than $\alpha>1$ when $\epsilon$ is large. 
This advantage is obvious when the length of the partition is large.
Besides, we observe that $\alpha < 1$ dominates earlier when $\theta=10$ than when $\theta=5$.
The reason is that when the partition is relatively large, it is difficult for the randomly generated query ranges to cover the complete partitions, which causes the order preservation of the values in the partition to have a greater impact on the results.
}

\noindent\textbf{Performance of OpBoost for Tree Boosting Tasks.}
We then evaluate the performance of OpBoost for tree boosting tasks.
We run both GBDT and XGBoost in OpBoost utilizing different order-preserving desensitization algorithms.
Binary classification tasks, multi-classification tasks, and regression tasks are all covered.
As before, the Piecewise mechanism is also conducted for comparison.
Considering that data pre-processing can also cause some loss of prediction accuracy to the trained model, we also compare the prediction accuracy of models trained by the pre-processed plaintext dataset and the plaintext dataset without any processing.

\ifshaphered
\emph{(Reviewer2:D1)}
\else
\fi
\revision{
\textbf{\it Prediction Accuracy of GBDT and XGBoost.}
Figure \ref{fig:adult_class} shows the results of GBDT running in OpBoost for binary classification task on Adult datasets.
Since more than $75\%$ samples are labeled \emph{Negative}, thus high prediction accuracy can be obtained as long as the model tends to output \emph{Negative}.
Therefore, the results in Figure \ref{fig:adult_class} show that adding noise to the dataset does not significantly reduce the prediction accuracy of the model.
Compared with other algorithms, Local-map always maintains obvious advantages when $\epsilon$ is small, i.e., $\epsilon<2.56$.
We think it is because the values desensitized by other algorithms are more likely to fall far away from raw values when $\epsilon$ is small, while Local-map keeps the desensitized values within the partitions where the raw value are located.
The desensitized values are kept in the partitions of raw values with a high probability when $\epsilon$ is large, so Adj-map that allocates more privacy budget to the $\epsilon_{ner}$ can be more dominant.
This also explains why a large $\alpha$ is better for Adj-map when $\epsilon<2.56$, and a small $\alpha$ is better when $\epsilon>2.56$.
Besides, we observe that the effect of $\theta$ on the results is not obvious, and $\theta$ and $\alpha$ both affect $\epsilon_{ner}$ and $\epsilon_{prt}$ in the same direction according to Equation \ref{equ:fit_base}.
Similar observations are also obtained from the results of the multi-classification and regression tasks.
Figure \ref{fig:pendigits_class} shows the results of GBDT running in OpBoost for multi-classification task on Pen-digits dataset.
Figure \ref{fig:plant_regress} and Figure \ref{fig:casp_regress} (deferred to Appendix \ref{app:figure_casp} due to limited space) show the results of regression task on Powerplant dataset and CASP dataset.
All proposed order-preserving desensitization algorithms outperform the Piecewise mechanism in accuracy in all tasks.
}

\ifshaphered
\emph{(Reviewer1:D2.a,Reviewer2:D1,Reviewer3:D1)}
\else
\fi
\revision{
We also conduct all tasks on XGBoost with the same datasets as GBDT.
Due to limited space, we defer the results to Appendix \ref{app:xgboost_result}.
Besides, we test OpBoost with XGBoost on an industrial large-scale regression dataset.
Different from other previous benchmark datasets, this industrial dataset contains a large number of categorical features with no order between values.
We use one-hot encoding to encode all categorical values and then randomize them with Unary Encoding (UE) \cite{wang2017locally}.
The results are summarized in Table \ref{tab:umeng_opxgboost}.
We show the ratio of the MSE of models trained with the proposed algorithms and the model trained by the unprocessed plaintext dataset when $\theta=2$.
The observations of results are all consistent with GBDT.
Note that the prediction accuracy of the model trained by dataset desensitized by Local-map is even higher than that of the plain model.
Such phenomenon is also observed in other tree boosting frameworks \cite{li2020practical, li2020privacy, tian2020federboost}.
A possible reason is that the desensitization introduces a certain amount of randomness when $\epsilon$ is large, which functions as a source of regularization and eventually improves the model generalization.
Some existing DP works study this, although for different DP mechanisms and ML models \cite{yeom2018privacy, dwork2015preserving, khatri2017preventing}.
Besides, we also show in Theorem \ref{the:max_gain} that the maximum gain of the desensitized feature is not lower than the raw feature with a certain probability. 
}

\ifshaphered
\emph{(Reviewer3:D2)}
\else
\fi
\revision{
\textbf{\it Communication and Computation overhead of GBDT and XGBoost.}
We record the training time and communication overhead of running GBDT and XGBoost in OpBoost on each dataset.
The results of large-scale dataset are shown in Table \ref{tab:umeng_opxgboost}, and the results of other datasets are shown in Appendix \ref{app:comun_comp}.
We find that the training time increases with the decrease of $\epsilon$, because smaller $\epsilon$ incurs more times of resampling the bounded discrete Laplacian noise.
To verify such a trend, we further compare the training time of sampling with exponential mechanism and bounded discrete Laplace in the partitions when $\theta$ is small.
Moreover, we observe that as $\epsilon$ decreases, the number of split points that Party A needs to request from Party B decreases.
We think the reason is that desensitizing a feature may reduce the maximum gain of its optimal split point.
Therefore, the model is inclined to find the split point on features on the Party A side when features on Party B are desensitized with a small $\epsilon$.
}

\section{Discussion}
\subsection{Multiple Features Desensitization}
\ifshaphered
\emph{(Reviewer1:D1)}
\else
\fi
\revision{
Although we focus on desensitizing a single numerical feature when presenting the proposed order-preserving desensitization algorithms for easier exposition, 
these algorithms can also be extended to the multiple features case. For one example, we can split the privacy budget among the feature and apply the proposed algorithms for each feature with its share of the privacy budget. In existing LDP studies dealing with multidimensional numerical values, evenly splitting the privacy budget can achieve satisfactory performance \cite{nguyen2016collecting, duchi2018minimax, wang2019collecting, wang2021local, couchot2021random}. For the other example, we can also randomly sample some features and only allocate the privacy budget to the sampled features, while the un-sampled features are not disclosed to the aggregator.
The second exemplary solution can obtain higher utility than the first one \cite{nguyen2016collecting, wang2019collecting, wang2021local}.}

\revision{
After applying the proposed algorithms for each feature, it suffices to apply the composition theorem to ensure that the overall processing for multidimensional values satisfies dLDP with a given privacy budget. Although dLDP and partition-dLDP relaxes the definition of LDP based on the distance, they still satisfy the sequential composition theorem \cite{dwork2014algorithmic}, which are proved in Appendix \ref{app:composition}.
}

\subsection{Setting of privacy parameters }
\ifshaphered
\emph{(Reviewer1:D4)}
\else
\fi
\revision{
In summary, we propose three order-preserving desensitization algorithms: Global-map, Local-map, and Adj-map.
Local-map and Adj-map support setting different $\alpha$ and $\theta$ when given a privacy budget $\epsilon$.
All three algorithms are contained in OpBoost.
Next, we give some guidelines on choosing these algorithms when using OpBoost in actual training tasks.
}

\revision{
(1) \emph{Utility prioritized users.} For all tasks, Local-map is always the best choice for users with less concern about the indistinguishability of values in different partitions.
(2) \emph{Privacy prioritized users.}
The users without extra domain knowledge and who do not pursue high accuracy of the trained model could directly choose Global-map to avoid considering parameters' settings other than $\epsilon$.
For users with extra domain knowledge and higher utility pursuit, they can choose Adj-map and set $\alpha$ and $\theta$ to obtain higher model accuracy than Global-map.
The length of the partition $\theta$ is usually determined by the semantic information of the features in practice.
We show in subsection \ref{sub:error_split} that the distribution and density of feature values and the values of labels all affect the maximum gain when finding the split point.
Besides, the training dataset usually contains multiple features, which together affect the trained model.
The distribution and density of each feature are different, and the influence on the trained model is also different.
Therefore, it's difficult to theoretically give an optimal parameter setting. 
But we can empirically give some suggestions on settings of $\alpha$ based on our experimental evaluation.
In all our experimental datasets, $\alpha<1$ is dominant only when $\epsilon$ is large, i.e., $\epsilon>2.56$.
As $\epsilon$ should usually be set less than $1$ or even $0.1$ in practical applications, users can always set $\alpha>1$.
Besides, the results show that when $\alpha=10$, there is no obvious improvement, so it is not necessary to set $\alpha$ too large.
}

\section{Related Works}
The extended version of related works is shown in Appendix \ref{app:related}.

\noindent \textbf{Distance-based LDP (dLDP).} The traditional DP mechanism always considers the worst case, which leads to adding excessive noise for normal cases.
Kifer et al. \cite{kifer2012rigorous} propose a semantic framework called "Pufferfish", which can generate customized privacy definitions in different scenarios. 
Geng et al. \cite{geng2015staircase} propose staircase mechanism to guarantee diverse levels of differential privacy for different instances.
\ifshaphered
\emph{(Reviewer3:D3)}
\else
\fi
\revision{The formal dLDP definition is first proposed and applied in Location-Based Systems to guarantee location privacy within a specific distance \cite{andres2013geo, xiao2015protecting}. }
Following the intuition of $d_\chi$-privacy in \cite{chatzikokolakis2013broadening}, Alvim et al. \cite{alvim2018local} define Metric-LDP, a variant of dLDP.
Afterward, dLDP shows its broad applicability in vast scenarios \cite{xiang2020linear, gursoy2019secure, chatzikokolakis2017efficient, shokri2014privacy, wang2017local, borgs2018revealing}.
However, the potential capabilities of dLDP in ordinal information preserving remain undiscussed.

\noindent \textbf{Privacy-Preserving Tree Boosting on Vertical FL.} 
Traditional tree boosting algorithms have drawn privacy concerns for their direct access to raw datasets.
In SecureBoost proposed by Cheng et al. \cite{cheng2021secureboost}, the parties exchange gradients and hessians encrypted with HE, which is extremely time-consuming.
Although the MPC-based schemes\cite{abspoel2021secure, wu2020privacy} avoid complex cryptographic operations, the massive communication overhead caused by MPC is unbearable.
To solve this problem, Tian et al. \cite{tian2020federboost} propose a scheme based on Local Differential Privacy (LDP).
Since the randomness introduced by LDP, the accuracy of the trained model is not satisfying.

\section{Conclusion}
In this paper, a novel framework called OpBoost is proposed for privacy-preserving vertical federated tree boosting. 
The privacy notion of dLDP is firstly applied in vertical federated tree boosting tasks.
It is shown that the prediction accuracy of the model trained by OpBoost is much higher than that of the LDP-based scheme.
Meanwhile, the computational and communication overheads of OpBoost are significantly lower than cryptography-based schemes.
Specifically, we optimize the existing dLDP definition and instantiate three order-preserving desensitization algorithms for OpBoost.
We also study and apply (bounded) discrete Laplace distribution as an alternative sampling distribution, which further reduces the computational overhead of the proposed algorithms.
Finally, we conduct a comprehensive evaluation to show the effectiveness and efficiency of OpBoost.

\balance
\bibliographystyle{ACM-Reference-Format}
\bibliography{sample}

\newpage
\begin{appendix}

\section{Proofs of Theorem}
\subsection{Proof of Theorem \ref{the:base}}
\label{app:base}
\begin{proof}
    Assume that $x, x'\in{\mathbb{D}_\bot}$ are two sensitive values. 
    And w.l.o.g, we let $x'-x=t$.
    Then the probability ratio of $x$ and $x'$ being randomized to the same output value is
    \begin{align*}
    \frac{Pr[O=o|x]}{Pr[O=o|x']}&=\frac{e^{-|x-o|\cdot\epsilon/2}}{\sum_{j\in\mathbb{D}_{\bot}}e^{-|x-j|\cdot\epsilon/2}}\cdot \frac{\sum_{j\in\mathbb{D}_{\bot}}e^{-|x'-j|\cdot\epsilon/2}}{e^{-|x'-o|\cdot\epsilon/2}}\\
    &=\frac{e^{-|x-o|\cdot\epsilon/2}}{e^{-|x+t-o|\cdot\epsilon/2}}\cdot \frac{\sum_{j\in\mathbb{D}_{\bot}}e^{-|x+t-j|\cdot\epsilon/2}}{\sum_{j\in\mathbb{D}_{\bot}}e^{-|x-j|\cdot\epsilon/2}}\\
    &\le e^{t\cdot\epsilon/2}\cdot e^{t\cdot\epsilon/2}=e^{t\cdot\epsilon}.
    \end{align*}
 \end{proof}

\subsection{Proof of Theorem \ref{the:baseuti}}
\label{app:uti-global-map}
\begin{proof}
  Let $x_1,x_2\in\mathbb{D}_\bot$ be two sensitive values, and $x_1<x_2$, $x_2-x_1=t$.
  We calculate the probability that the output value $o_2$ of $x_2$ is greater than the output value $o_1$ of $x_1$.
  We have
  \begin{align*}
    Pr[o_2>o_1]&=\sum_{o_2\in[L,R]}\sum_{o_1\in[L,o_2)} Pr[O=o_1|x_1]\cdot Pr[O=o_2|x_2]\\
    &=\sum_{o_2\in (L,x_1]}\sum_{o_1\in [L,o_2)}Pr[O=o_1|x_1]\cdot Pr[O=o_2|x_2]\\
    &+\sum_{o_2\in (x_1,x_2]}\sum_{o_1\in [L,x_1)}Pr[O=o_1|x_1]\cdot Pr[O=o_2|x_2]\\
    &+\sum_{o_2\in (x_1,x_2]}\sum_{o_1\in [x_1,o_2)}Pr[O=o_1|x_1]\cdot Pr[O=o_2|x_2]\\
    &+\sum_{o_2\in (x_2,R]}\sum_{o_1\in [L,x_1)}Pr[O=o_1|x_1]\cdot Pr[O=o_2|x_2]\\
    &+\sum_{o_2\in (x_2,R]}\sum_{o_1\in [x_1,o_2)}Pr[O=o_1|x_1]\cdot Pr[O=o_2|x_2]\\
  \end{align*}
  For brevity, we denote $q=e^{-\epsilon/2}$, $D=((1+q-q^{x_1-L+1}-q^{R-x_1+1})\cdot(1+q-q^{x_2-L+1}-q^{R-x_2+1}))^{-1}$.
  \begin{align*}
    &\sum_{o_2\in (L,x_1]}\sum_{o_1\in [L,o_2)}Pr[O=o_1|x_1]\cdot Pr[O=o_2|x_2]\\
    &=\sum_{o_2\in (L,x_1]}\sum_{o_1\in [L,o_2)}\frac{q^{|o_1-x_1|}}{\sum_{v_1\in [L,R]}q^{|v_1-x_1|}}\cdot \frac{q^{|o_2-x_2|}}{\sum_{v_2\in [L,R]}q^{|v_2-x_2|}}\\
    &=D\cdot (1-q)^2\cdot \sum_{o_2\in (L,x_1]}\sum_{o_1\in [L,o_2)}q^{x_2-o_2}\cdot q^{x_1-o_1}\\
    &=D\cdot(q^{x_1+x_2-2L+1}-q^{x_2-L+1}+\frac{q^{x_2-x_1+1}-q^{x_1+x_2-2L+1}}{1+q})\\
    &=D\cdot(\frac{q^{x_1+x_2-2L+2}+q^{x_2-x_1+1}}{1+q}-q^{x_2-L+1})
  \end{align*}

  Similarly, we have
  \begin{align*}
      &\sum_{o_2\in (x_1,x_2]}\sum_{o_1\in [L,x_1)}Pr[O=o_1|x_1]\cdot Pr[O=o_2|x_2]\\
      &=D\cdot(q-q^{x_1-L+1}-q^{x_2-x_1+1}+q^{x_2-L+1})\\
  \end{align*}
  \begin{align*}
      &\sum_{o_2\in (x_1,x_2]}\sum_{o_1\in [x_1,o_2)}Pr[O=o_1|x_1]\cdot Pr[O=o_2|x_2]\\
      &=D\cdot(1-((1-q)(x_2-x_1)+1)\cdot q^{x_2-x_1})\\
  \end{align*}
  \begin{align*}
      &\sum_{o_2\in (x_2,R]}\sum_{o_1\in [L,x_1)}Pr[O=o_1|x_1]\cdot Pr[O=o_2|x_2]\\
      &=D\cdot(q^2-q^{x_1-L+2}-q^{R-x_2+2}+q^{R-x_2+x_1-L+2})\\
  \end{align*}
  \begin{align*}
      &\sum_{o_2\in (x_2,R]}\sum_{o_1\in [x_1,o_2)}Pr[O=o_1|x_1]\cdot Pr[O=o_2|x_2]\\
      &=D\cdot(q-q^{R-x_2+1}-\frac{q^{x_2-x_1+2}+q^{2R-x1-x2+2}}{1+q})
  \end{align*}
  
  By summation, we have
  \begin{align*}
    &Pr[o_2>o_1] = D\cdot((1+q)^2+\frac{q^{x_1+x_2-2L+2}+q^{2R-x_1-x_2+2}}{1+q}+\\
    &\qquad\qquad\qquad\qquad q^{R-L-t+2}-q(q+1)(q^{x_1-L}-q^{R-x_2}))
  \end{align*}
  
  $Pr[o_2>o_1]$ can be regard as a function of $x_1$, and $Pr[o_2>o_1]$ minimizes at $x_1=L$ or $x_1=R-t$.
  Without loss of generality, we let $x_1=L$
  \begin{align*}
    &Pr[o_2>o_1]\\
    &\ge \frac{(1+q)^2+\frac{1}{1+q}(q^{t+2}+q^{2R-2L-t+2})+q^{R-L-t+2}}{(1+q-q-q^{R-L+1})\cdot(1+q-q^{t+1}-q^{R-L-t+1})}-\\
    &\qquad\qquad \frac{q(q+1)(1+q^{R-L-t})+(\frac{2q^2+q+1}{1+q}+(1-q)t)\cdot q^t}{(1+q-q-q^{R-L+1})\cdot(1+q-q^{t+1}-q^{R-L-t+1})}\\
    &=\frac{1+q+\frac{q^{2R-2L-t+2}}{1+q}-q^{R-L-t+1}-(\frac{q^2+q+1}{1+q}+(1-q)t)\cdot q^t}{(1-q^{R-L+1})\cdot(1+q-q^{t+1}-q^{R-L-t+1})}\\
    &\ge\frac{1+q-q^{R-L-t+1}-(\frac{q^2+q+1}{1+q}+(1-q)t)\cdot q^t}{1+q-q^{t+1}-q^{R-L-t+1}}\\
    &=1-\frac{\frac{1}{1+q}+(1-q)\cdot t}{1+q-q^{t+1}-q^{|D|-t}}\cdot q^t
  \end{align*}
  Thus, 
  \begin{align*}
      Pr[o_2>o_1]\ge 1-\frac{(1-q^2)\cdot t+1}{(1+q-q^{t+1}-q^{|\mathbb{D}_{\bot}|-t})(1+q)}\cdot q^t
  \end{align*}
\end{proof}

\subsection{Proof of Theorem \ref{the:ldputi}}
\label{app:uti-grr}
\begin{proof}
      Assume that $x_1, x_2\in{\mathbb{D}_\bot}$ are two sensitive values, $o_1, o_2\in{\mathbb{D}_\bot}$ are their corresponding output. And w.l.o.g, we let $x_2-x_1=t$.
       The perturbation mechanism of GRR is 
       \begin{equation}
           Pr[GRR(x)=o]= \begin{cases}
             p_1=\frac{e^{\epsilon}}{|\mathbb{D}_{\bot}|+e^{\epsilon}-1}, o=x\\
             p_2=\frac{1}{|\mathbb{D}_{\bot}|+e^{\epsilon}-1}, o\neq x\\
           \end{cases}
       \end{equation}
      We have
      \begin{align*}
       &Pr[o_2>o_1|x_2>x_1]=Pr[o_1=x_1 \wedge o_2=x_2 ]
       \\
       &+Pr[o_1=x_1 \wedge o_2>x_1 \wedge o_2 \neq x_2]
       +Pr[o_2=x_2 \wedge o_1<x_2 \wedge o_1 \neq x_1]
       \\
       &+Pr[o_1<o_2 \wedge o_1 \neq x_1 \wedge o_2 \neq x_2]
       \\
       &=p_1^2+p_1p_2(R-x_1-1)+p_1p_2\cdot (x_2-L-1)
       \\
       &+p_2^2\cdot(\frac{1}{2}\cdot|\mathbb{D}_{\bot}|(|\mathbb{D}_{\bot}|-1)-(R-x_1-1)-(x_2-L-1)-1)
       \\
       &=p_1^2+p_1p_2\cdot(|\mathbb{D}_{\bot}|+t-3)+p_2^2\cdot(\frac{1}{2}|\mathbb{D}_{\bot}|(|\mathbb{D}_{\bot}|-3)-t+2)
       \\
       &=p_1^2+p_1p_2\cdot(|\mathbb{D}_{\bot}|-3)+p_2^2\cdot(\frac{1}{2}|\mathbb{D}_{\bot}|(|\mathbb{D}_{\bot}|-3)+2) + p_2(p_1-p_2)t
       \end{align*}
 \end{proof}

\subsection{Proof of Theorem \ref{the:adj-map}}
\label{app:adj_map}
\begin{proof}
   We assume that $x$ and $x'$ are two values with $|x-x'|\le t$.
   The partition to which $x$ belongs is $\mathcal{P}_i$ and $x'$ belongs to partition $\mathcal{P}_j$, where $\mathcal{P}_i$ and $\mathcal{P}_j$ are at most $\lceil\frac{t}{\theta}\rceil$ partitions apart.
   Denote $M(x)$ as Algorithm \ref{alg:rand-map} with input $x$, we first compute the probability ratio of $M(x)$ and $M(x')$ are in the same partition $\mathcal{P}_{\hat m}$ is
	\begin{align*}
		&\frac{Pr[M(x)=\mathcal{P}_{\hat m}|x\in \mathcal{P}_i]}{Pr[M(x')=\mathcal{P}_{\hat m}|x'\in \mathcal{P}_j]}\\
    &=\frac{e^{-|i-\hat m|\cdot\frac{\epsilon_{prt}}{2}}}{\sum_{v\in[k]}e^{-|i-v|\cdot\frac{\epsilon_{prt}}{2}}}\cdot\frac{\sum_{v\in[k]}e^{-|j-v|\cdot\frac{\epsilon_{prt}}{2}}}{e^{-|j-\hat m|\cdot\frac{\epsilon_{prt}}{2}}}\\
    &=e^{\frac{\epsilon_{prt}}{2}\cdot(|j-\hat m|-|i-\hat m|)}\cdot\frac{\sum_{v\in[k]}e^{-|j-v|\cdot\frac{\epsilon_1}{2}}}{\sum_{v\in[k]}e^{-|i-v|\cdot\frac{\epsilon_1}{2}}}\\
    &\le e^{|j-i|\cdot\frac{\epsilon_{prt}}{2}}\cdot e^{|j-i|\cdot\frac{\epsilon_{prt}}{2}}
    \le e^{\lceil\frac{t}{\theta}\rceil\cdot\epsilon_{prt}}.
	\end{align*}
	For $x$ and $x'$ mapped to the same partition, the maximum probability ratio of $x$ and $x'$ being desensitized to the same output value can be obtained when $x$ and $x'$ fall respectively on the left and right sides outside the partition.
  It satisfies
	\begin{align*}
    Pr[O=o|M(x)=&\mathcal{P}_{\hat m}, x\in \mathcal{P}_i]\\
    &\le e^{\theta\cdot\epsilon_{ner}}Pr[O=o|M(x')=\mathcal{P}_{\hat m},x'\in \mathcal{P}_j].
	\end{align*}
    Therefore, we have
	\begin{align*}
		Pr[O=o|x\in \mathcal{P}_i]&=\sum_{\mathcal{P}_{\hat m}\owns o}Pr[M(x)=\mathcal{P}_{\hat m}|x\in \mathcal{P}_i]\cdot\\
    &\qquad \qquad \qquad Pr[O=o|M(x)=\mathcal{P}_{\hat m},x\in \mathcal{P}_i]\\
    &\le \sum_{\mathcal{P}_{\hat m}\owns o}e^{\lceil\frac{t}{\theta}\rceil\epsilon_{prt}}\cdot Pr[M(x')=\mathcal{P}_{\hat m}|x'\in \mathcal{P}_j]\cdot\\
    &\qquad \ \ \ \ \  e^{\theta\epsilon_{ner}}\cdot Pr[O=o|M(x')=\mathcal{P}_{\hat m},x'\in \mathcal{P}_j]\\
		&=e^{\lceil\frac{t}{\theta}\rceil\epsilon_{prt}+\theta\epsilon_{ner}}Pr[O=o|x'\in \mathcal{P}_j].
	\end{align*}
\end{proof}
\subsection{Proof of Theorem \ref{the:randmaputi}}
\label{app:uti-adj-map}
\begin{proof}
  Let $x_1,x_2\in\mathbb{D}$ be two sensitive data points.
  Specifically, $x_1<x_2$, and $t=x_2-x_1$, $T=\lfloor\frac{t}{\theta}\rfloor$.
  We calculate the probability that the output value $o_2$ of $x_2$ is greater than the output value $o_1$ of $x_1$.
  Denote $\hat{\mathcal{P}}_{m(x_1)}$ as the partition that $x_1$ is mapped, $\hat{\mathcal{P}}_{m(x_2)}$ as the partition that $x_2$ is mapped.
  If $\hat{\mathcal{P}}_{m(x_1)}$ and $\hat{\mathcal{P}}_{m(x_2)}$ are two different partitions and $\hat{\mathcal{P}}_{m(x_1)}$ is on the left of $\hat{\mathcal{P}}_{m(x_2)}$, then it has
  \[Pr[o_2>o_1|\hat{\mathcal{P}}_{m(x_2)}>\hat{\mathcal{P}}_{m(x_1)}]=1\]
  According to the result of Theorem \ref{the:baseuti}, we can directly get the probability that $\hat{\mathcal{P}}_{m(x_1)}$ is on the left of $\hat{\mathcal{P}}_{m(x_2)}$ as
  \[Pr[\hat{\mathcal{P}}_{m(x_2)}>\hat{\mathcal{P}}_{m(x_1)}]\ge 1-\frac{(1-q^2)\cdot T+1}{(1+q-q^{T+1}-q^{k-T})(1+q)}\cdot q^T\]
  The probability that $x_1$ and $x_2$ are mapped to the same partition is
  \begin{align*}
      &Pr[\hat{\mathcal{P}}_{m(x_2)}=\hat{\mathcal{P}}_{m(x_1)}]\\
      &=\sum_{\mathcal{P}_o\in [\mathcal{P}_1,\mathcal{P}_k]} Pr[RM(x_1)=\mathcal{P}_o]\cdot Pr[RM(x_2)=\mathcal{P}_o]\\
      &=\sum_{\mathcal{P}_o\in [[\mathcal{P}_1,\mathcal{P}_k]]} \frac{q^{|m(x_1)-o|}}{\sum_{\mathcal{P}_c\in [\mathcal{P}_1,\mathcal{P}_k]}q^{|m(x_1)-c|}}\cdot \frac{q^{|m(x_2)-o|}}{\sum_{\mathcal{P}_c\in [\mathcal{P}_1,\mathcal{P}_k]}q^{|m(x_2)-c|}}\\
  \end{align*}
   By calculating the above probability summation formula, we can get
   \[Pr[\hat{\mathcal{P}}_{m(x_2)}=\hat{\mathcal{P}}_{m(x_1)}]\ge \frac{(1-q)^2(T+1)}{(1+q)^2}\cdot q^T\]
   Since $x_2>x_1$, when $o_1$ and $o_2$ are in the same partition, it has
   \[Pr[o_2>o_1|\hat{\mathcal{P}}_{m(x_2)}=\hat{\mathcal{P}}_{m(x_1)}]>Pr[o_1>o_2|\hat{\mathcal{P}}_{m(x_2)}=\hat{\mathcal{P}}_{m(x_1)}]\]
   So we can approximate the probability $Pr[o_2>o_1,\hat{\mathcal{P}}_{m(x_2)}=\hat{\mathcal{P}}_{m(x_1)}]$ as
   \[Pr[o_2>o_1,\hat{\mathcal{P}}_{m(x_2)}=\hat{\mathcal{P}}_{m(x_1)}]\ge\frac{(1-q)^2(T+1)}{2(1+q)^2}\cdot q^T\]
   Finally, we have
   \begin{align*}
       &Pr[o_2>o_1]\\
       &=Pr[o_2>o_1,\hat{\mathcal{P}}_{m(x_2)}=\hat{\mathcal{P}}_{m(x_1)}]+Pr[o_2>o_1,\hat{\mathcal{P}}_{m(x_2)}>\hat{\mathcal{P}}_{m(x_1)}]\\
       &\ge 1-\frac{((1-q^2)\cdot T+1)\cdot q^T}{(1+q-q^{T+1}-q^{k-T})(1+q)}+\frac{(1-q)^2(T+1)\cdot q^T}{2(1+q)^2}
   \end{align*}
\end{proof}

\subsection{Proof of Theorem \ref{the:pribounded-laplace}}
\label{app:pribounded-laplace}
\begin{proof}
    Let $v_1$ and $v_2$ are two values with $v_1-v_2=t$, $[l,u]\subseteq\mathbb{D}_{\bot}$ is the range of output value $o$ after adding bounded discrete Laplace noise.
    Denote $N_1$ and $N_2$ as the random noise sampling from bounded discrete Laplace noise $Lap_{\mathbb{Z}}(\frac{1}{\epsilon})$, where $N_1\in[l-v_1,u-v_1]$ and $N_2\in[l-v_2,u-v_2]$.
    We prove that
    \[\frac{Pr[v_1+N_1=o]}{Pr[v_2+N_2=o]}=\frac{Pr[N_1=o-v_1]}{Pr[N_2=o-v_2]}\le e^{2t\epsilon}\]
    When $v_1<v_2<l<u$, it has $l-v_1>0$, $u-v_1>0$, $l-v_2>0$, and $u-v_2>0$.
    Then we have the ratio of $Pr[N_1=o-v_1]$ and $Pr[N_2=o-v_2]$ is
    \begin{align*}
    \frac{Pr[N_2=o-v_2]}{Pr[N_1=o-v_1]}&=\frac{e^{-(l-v_1)\epsilon}(1-e^{-(u-l+1)\epsilon})}{e^{-(l-v_2)\epsilon}(1-e^{-(u-l+1)\epsilon})}\cdot\frac{e^{-|o-v_2|\epsilon}}{e^{-|o-v_1|\epsilon}}\\
    &=e^{((o-v_1)-(o-v_2)\epsilon}\cdot e^{((l-v_2)-(l-v_1))\epsilon}\\
    &=e^{(v_2-v_1)\epsilon}\cdot e^{(v_1-v_2)\epsilon}=1
    \end{align*}
    When $l<u<v_1<v_2$, it has $l-v_1<0$, $u-v_1<0$, $l-v_2<0$, and $u-v_2<0$.
    Then we have
    \begin{align*}
    \frac{Pr[N_1=o-v_1]}{Pr[N_2=o-v_2]}&=\frac{e^{(u-v_2)\epsilon}(1-e^{-(u-l+1)\epsilon})}{e^{(u-v_1)\epsilon}(1-e^{-(u-l+1)\epsilon})}\cdot\frac{e^{-|o-v_1|\epsilon}}{e^{-|o-v_2|\epsilon}}\\
    &=e^{((u-v_2)-(u-v_1))\epsilon}\cdot e^{(-(o-v_2)+(o-v_1))\epsilon}\\
    &=e^{(v_1-v_2)\epsilon}\cdot e^{(v_2-v_1)\epsilon}=1
    \end{align*}
    When $l<v_1<v_2<u$, it has $l-v_1<0$, $u-l>0$, $l-v_2<0$, and $u-v_2>0$.
    Then we have
    \begin{align*}
    &\frac{Pr[N_1=o-v_1]}{Pr[N_2=o-v_2]}\\
    &=\frac{1-e^{-(-(l-v_2)+1)\epsilon}-e^{-(u-v_2+1)\epsilon}+e^{-\epsilon}}{1-e^{-(-(l-v_1)+1)\epsilon}-e^{-(u-v_1+1)\epsilon}+e^{-\epsilon}}\cdot\frac{e^{-|o-v_1|\epsilon}}{e^{-|o-v_2|\epsilon}}\\
    &\le e^{t\epsilon}\cdot e^{t\epsilon}\cdot\frac{e^{-t\epsilon}-e^{-(-(l-v_2)+1+t)\epsilon}-e^{-(u-v_2+1+t)\epsilon}+e^{-(1+t)\epsilon}}{1-e^{-(-(l-v_1)+1)\epsilon}-e^{-(u-v_1+1)\epsilon}+e^{-\epsilon}}\\
    &\le e^{2t\epsilon}
    \end{align*}
    When $v_1<l<v_2<u$, it has $l-v_1>0$, $u-v_1>0$, $l-v_2<0$, and $u-v_2>0$.
    Then we have
    \begin{align*}
    &\frac{Pr[N_2=o-v_2]}{Pr[N_1=o-v_1]}\\
    &=\frac{e^{-|o-v_2|\epsilon}}{e^{-|o-v_1|\epsilon}}\cdot\frac{e^{-(l-v_1)\epsilon}(1-e^{-(u-l+1)\epsilon})}{1-e^{-(-(l-v_2)+1)\epsilon}-e^{-(u-v_2+1)\epsilon}+e^{-\epsilon}}\\
    &=e^{((o-v_1)-|o-v_2|)\epsilon}\cdot \frac{e^{-(l-v_1)\epsilon}(1-e^{-(u-l+1)\epsilon})}{1-e^{((l-v_2)-1)\epsilon}-e^{-(u-v_2+1)\epsilon}+e^{-\epsilon}}\\
    &\le e^{t\epsilon}\cdot\frac{e^{-(l-v_1)\epsilon}}{e^{(l-v_2)\epsilon}}\cdot\frac{1-e^{-(u-l+1)\epsilon}}{e^{-(l-v_2)\epsilon}-e^{-\epsilon}-e^{(2v_2-(u+l)-1)\epsilon}+e^{(v_2-l-1)\epsilon}}\\
    &=e^{t\epsilon}\cdot e^{(v_1+v_2-2l)\epsilon}\cdot\frac{1-e^{-(u-l+1)\epsilon}}{e^{-(l-v_2)\epsilon}-e^{-\epsilon}-e^{2v_2-(u+l)-1}+e^{(v_2-l-1)\epsilon}}\\
    &\le e^{2t\epsilon}\cdot\frac{1-e^{-(u-l+1)\epsilon}}{e^{-(l-v_2)\epsilon}-e^{-\epsilon}-e^{(2v_2-(u+l)-1)\epsilon}+e^{(v_2-l-1)\epsilon}}\\
    &\le e^{2t\epsilon}\cdot\frac{1-e^{-(u-l+1)\epsilon}}{1-e^{-(u-l+1)\epsilon}}=e^{2t\epsilon}
    \end{align*}
    When $l<v_1<u<v_2$, it has $l-v_1<0$, $u-v_1>0$, $l-v_2<0$, and $u-v_2<0$.
    Then we have
    \begin{align*}
    &\frac{Pr[N_1=o-v_1]}{Pr[N_2=o-v_2]} \\
    &=\frac{e^{-|o-v_1|\epsilon}}{e^{-|o-v_2|\epsilon}}\cdot\frac{e^{(u-v_2)\epsilon}(1-e^{-(u-l+1)\epsilon})}{1-e^{-(-(l-v_1)+1)\epsilon}-e^{-(u-v_1+1)\epsilon}+e^{-\epsilon}}\\
    &=e^{((v_2-o)-|o-v_1|)\epsilon}\cdot \frac{e^{(u-v_2)\epsilon}(1-e^{-(u-l+1)\epsilon})}{1-e^{((l-v_1)-1)\epsilon}-e^{-(u-v_1+1)\epsilon}+e^{-\epsilon}}\\
    &\le e^{t\epsilon}\cdot\frac{e^{(u-v_2)\epsilon}}{e^{-(u-v_1)\epsilon}}\cdot\frac{1-e^{-(u-l+1)\epsilon}}{e^{(u-v_1)\epsilon}-e^{(u+l-2v_1-1)\epsilon}-e^{-\epsilon}+e^{(u-v_1-1)\epsilon}}\\
    &=e^{t\epsilon}\cdot e^{(2u-v_1-v_2)\epsilon}\cdot\frac{1-e^{-(u-l+1)\epsilon}}{e^{(u-v_1)\epsilon}-e^{(u+l-2v_1-1)\epsilon}-e^{-\epsilon}+e^{(u-v_1-1)\epsilon}}\\ 
    &\le e^{2t\epsilon}\cdot\frac{1-e^{-(u-l+1)\epsilon}}{e^{(u-v_1)\epsilon}-e^{(u+l-2v_1-1)\epsilon}-e^{-\epsilon}+e^{(u-v_1-1)\epsilon}}\\
    &\le e^{2t\epsilon}\cdot\frac{1-e^{-(u-l+1)\epsilon}}{1-e^{-(u-l+1)\epsilon}}=e^{2t\epsilon}
    \end{align*}
    In summary, for any output range $[l,u]\subseteq\mathbb{D}_{\bot}$, the ratio of $Pr[N_1=o-v_1]$ and $Pr[N_2=o-v_2]$ satisfies
    \[\frac{Pr[v_1+N_1=o]}{Pr[v_2+N_2=o]}=\frac{Pr[N_1=o-v_1]}{Pr[N_2=o-v_2]}\le e^{2t\epsilon}\]
\end{proof}

\subsection{Proof of Lemma \ref{lem:bounded-laplace}}
\label{app:distribution_bdldp}
\begin{proof}
  According to the definition of discrete Laplace in Definition \ref{def:discrete-laplace}, we have
  \begin{align*}
      &\sum_{z\in\mathbb{Z}}\frac{e^{1/\lambda}-1}{e^{1/\lambda}+1}\cdot e^{-|z|/\lambda}=\sum_{z\in[l,u]}\tau\cdot\frac{e^{1/\lambda}-1}{e^{1/\lambda}+1}\cdot e^{-|z|/\lambda}=1\\
      \Rightarrow &\tau=\frac{\sum_{z\in\mathbb{Z}}e^{-|z|/\lambda}}{\sum_{z\in[l,u]}e^{-|z|/\lambda}}.
  \end{align*}
  Here, we have
  \begin{align*}
      &\sum_{z\in\mathbb{Z}}e^{-|z|/\lambda}=\frac{2}{1-e^{-1/\lambda}}\cdot\lim_{n\rightarrow\infty}1-e^{-n/\lambda}\simeq\frac{2}{1-e^{-1/\lambda}}.\\
      &\sum_{z\in[l,u]}e^{-|z|/\lambda}=
      \begin{cases}
        \frac{e^{u/\lambda}(1-e^{-(u-l+1)/\lambda})}{1-e^{-1/\lambda}}, &l<u<0,\\
        \frac{1-e^{(-l+1)/\lambda}}{1-e^{-1/\lambda}}+\frac{e^{-1/\lambda}(1-e^{-u/\lambda})}{1-e^{-1/\lambda}}, &l<0<u,\\
        \frac{e^{-l/\lambda}(1-e^{-(u-l+1)/\lambda})}{1-e^{-1/\lambda}}, &0<l<u.\\
      \end{cases}
  \end{align*}
  Finally, we can calculate $\tau$ to get the distribution.
\end{proof}
\revision{
\section{Proofs of Composition Theorem of dLDP and partition-dLDP}
\label{app:composition}
\begin{theorem}
  Let $M_i$ be an $\epsilon_i$-dLDP mechanism for $i\in[k]$.
  Then $M_{[k]}(x)=(M_1(x), ... , M_k(x))$ satisfies $(\sum_{i=1}^k \epsilon_i)-dLDP$.
\end{theorem}
\begin{proof}
  Let $x,\ y$ are two values with $|x-y|\le t$.
  Fix any $(r_1, ... , r_k)$ from output domain, then we have:
  \begin{align*}
    \frac{Pr[M_{[k](x)=(r_1,...,r_k)}]}{Pr[M_{[k](y)=(r_1,...,r_k)}]}&=\frac{Pr[M_1(x)=r_1]\cdot...\cdot Pr[M_k(x)=r_k]}{Pr[M_1(y)=r_1]\cdot...\cdot Pr[M_k(y)=r_k]}\\
    &=(\frac{Pr[M_1(x)=r_1]}{Pr[M_1(y)=r_1]})\cdot...\cdot(\frac{Pr[M_k(x)=r_k]}{Pr[M_k(x)=r_k]})\\
    &\le \prod_{i=1}^k e^{t\cdot\epsilon_i}=e^{t\cdot\sum_{i=1}^k\epsilon_i}
  \end{align*}
\end{proof}
\begin{theorem}
  Let $M_i$ be an $(\epsilon_{prt}^i, \epsilon_{ner}^i)$-partition-dLDP mechanism for $i\in[k]$.
  Then $M_{[k]}(x)=(M_1(x), ... , M_k(x))$ satisfies $(\sum_{i=1}^k\epsilon_{prt}^i, \sum_{i=1}^k\epsilon_{ner}^i)$-partition-dLDP.
\end{theorem}
\begin{proof}
  Let $x,\ y$ are two values with $|x-y|\le t$.
  The partition to which $x$ belongs is $\mathcal{P}_i$ and $x'$ belongs to partition $\mathcal{P}_j$, where $\mathcal{P}_i$ and $\mathcal{P}_j$ are at most $\lceil\frac{t}{\theta}\rceil$ partitions apart.
  Fix any $(r_1, ... , r_k)$ from output domain, then we have:
  \begin{align*}
    \frac{Pr[M_{[k](x)=(r_1,...,r_k)}]}{Pr[M_{[k](y)=(r_1,...,r_k)}]}&=\frac{Pr[M_1(x)=r_1]\cdot...\cdot Pr[M_k(x)=r_k]}{Pr[M_1(y)=r_1]\cdot...\cdot Pr[M_k(y)=r_k]}\\
    &=(\frac{Pr[M_1(x)=r_1]}{Pr[M_1(y)=r_1]})\cdot...\cdot(\frac{Pr[M_k(x)=r_k]}{Pr[M_k(x)=r_k]})\\
    &\le \prod_{i=1}^k e^{\lceil\frac{t}{\theta}\rceil\epsilon_{prt}^i+\theta\epsilon_{ner}^i}=e^{\lceil\frac{t}{\theta}\rceil\sum_{i=1}^k\epsilon_{prt}^i+\theta\sum_{i=1}^k\epsilon_{ner}^i}
  \end{align*}
\end{proof}
}
\section{EXTENDED RELATED WORKS}
\label{app:related}
\noindent \textbf{Distance-based LDP (dLDP).} The traditional DP mechanism always considers the worst case, which leads to adding excessive noise for normal cases.
Kifer et al. \cite{kifer2012rigorous} propose a semantic framework called "Pufferfish", which can generate customized privacy definitions in different scenarios. 
Inspired by the Pufferfish, He et al. \cite{he2014blowfish} initiate a policy to specify the concept of secrets and constraints, and formally introduce the definition of dDP.
Geng et al. \cite{geng2015staircase} propose staircase mechanism to guarantee diverse levels of differential privacy for different instances.
Nevertheless, the staircase mechanism fails to provide more sophisticated probability distribution within one partition.
Because LDP is not as dependent on trusted servers as DP, LDP mechanisms are more prevalent in practical applications.
\ifshaphered
\emph{(Reviewer3:D3)}
\else
\fi
\revision{The formal dLDP definition is first proposed and applied in Location-Based Systems to guarantee location privacy within a specific distance \cite{andres2013geo, xiao2015protecting}. }
Following the intuition of $d_\chi$-privacy in \cite{chatzikokolakis2013broadening}, Alvim et al. \cite{alvim2018local} define Metric-LDP, a variant of dLDP.
Afterward, dLDP shows its broad applicability in vast scenarios \cite{xiang2020linear, gursoy2019secure, chatzikokolakis2017efficient, shokri2014privacy, wang2017local, borgs2018revealing}.
However, the potential capabilities of dLDP in ordinal information preserving remain undiscussed.

\noindent \textbf{Privacy-Preserving Tree Boosting on Vertical FL.} 
Traditional tree boosting algorithms have drawn public privacy concerns for their direct access to raw datasets, which leads to the emergence of privacy-preserving tree boosting.
In SecureBoost proposed by Cheng et al. \cite{cheng2021secureboost}, the user holding labels send gradients and hessians encrypted with HE to other users for sorting.
Fu et al. \cite{fu2021vf2boost} proposed $VF^2$Boost to optimize SecureBoost from the perspective of engineering implementation.
However, the training process is still extremely time-consuming since a lot of cryptographic operations are irreducible.
A scheme based on Multi-Party Computation (MPC) is proposed by Abspoel et al. \cite{abspoel2021secure}, in which only the split points of the features' values are revealed in the whole training procedure.
Wu et al. \cite{wu2020privacy} design another MPC-based scheme that guarantees high security for users' records.
Although these schemes avoid complex cryptographic operations, the massive communication overhead caused by MPC is unbearable.
To solve this problem, Tian et al. \cite{tian2020federboost} propose a scheme based on Local Differential Privacy (LDP), called FederBoost.
Since the randomness introduced by LDP, the accuracy of the trained model is not satisfying.

\section{Supplementary Experiment}
\subsection{Computation overhead compared with OPE}
\label{app:ope_compare}
\ifshaphered
\emph{(Reviewer1:D2.c)}
\else
\fi
\revision{
To the best of our knowledge, there is no federated tree boosting scheme that uses Order-Preserving Encryption (OPE).
However, it can be implemented by replacing the desensitization algorithms in our framework with an OPE scheme.
In this experiment, we employ pyope 0.2.2 library for Boldyreva symmetric OPE scheme.
The only difference between the two frameworks is the data desensitization algorithms used while the rest is the same, so we only need to compare the computation time of the desensitization algorithms.
We randomly generate $100,000$ uniformly distributed data in the range $[1, 100]$, process them using our algorithms and OPE respectively, and record the computation time.
This experiment is conducted on a PC with Intel(R) Core(TM) i7-9700 CPU @ 3.00GHz and 32GB memory.
As shown in Table \ref{tab-computation-ope}, our desensitization algorithms are 40-200 times faster than OPE.
}
\begin{table}[h]
  \small
  \centering
  \begin{tabular}{|c|c|c|}
      \hline
      \diagbox{Method}{$\epsilon$}&\shortstack{$\epsilon=0.08$}&\shortstack{$\epsilon=1.28$}\\
      \hline
      OPE&\multicolumn{2}{c|}{241.9227s}\\
      \hline
      GLobal-map&{1.3874s}&{1.2764s}\\
      \hline
      Local-map$(\theta=4)$&{3.0140s}&{3.0189s}\\
      \hline
      Local-map$(\theta=10)$&{5.8893s}&{5.7287s}\\
      \hline
      Adj-map$(\theta=4, \alpha=0.4)$&{4.1220s}&{3.0685s}\\
      \hline
      Adj-map$(\theta=4, \alpha=1)$&{4.1060s}&{2.9861s}\\
      \hline
      Adj-map$(\theta=4, \alpha=10)$&{4.0763s}&{3.0060s}\\
      \hline
      Adj-map$(\theta=10, \alpha=0.4)$&{5.6929s}&{3.4936s}\\
      \hline
      Adj-map$(\theta=10, \alpha=1)$&{5.5541s}&{3.5944s}\\
      \hline
      Adj-map$(\theta=10, \alpha=10)$&{5.5308s}&{3.4848s}\\
      \hline
  \end{tabular}
  \setlength{\abovecaptionskip}{0ex}
  \setlength{\belowcaptionskip}{0ex}
  \caption{The computaion time of desensitizing $100,000$ values which follow uniform distribution within $[1, 100]$.}
  \label{tab-computation-ope}
\end{table}
\subsection{Theoretical Order-Preserving Probability}
\label{app:_theo_OPP}

We compare the theoretical order-preserving probability of proposed algorithms.
Although the lower bound of order-preserving probabilities $\gamma$ are formally deduced in Section \ref{sec:algorithms}, the formulas are complicated and challenging to interpret directly.
We visualize both the exhaustively accumulative probability and the derived lower bound of the probability for intuitive comparison.

\begin{table*}[t]
  \small
  \centering
  \begin{tabular}{|c|c|c|c|c|c|c|c|c|c|c|}
      \hline
      \diagbox{Method}{$dist$}&\shortstack{$dist=0$\\(t=5)}&\shortstack{$dist=1$\\(t=15)}&\shortstack{$dist=2$\\(t=25)}&\shortstack{$dist=3$\\(t=35)}&\shortstack{$dist=4$\\(t=45)}&\shortstack{$dist=5$\\(t=55)}&\shortstack{$dist=6$\\(t=65)}&\shortstack{$dist=7$\\(t=75)}&\shortstack{$dist=8$\\(t=85)}&\shortstack{$dist=9$\\(t=95)}\\
      \hline
      GRR$(\epsilon)$&{0.4950}&{0.4951}&{0.4953}&{0.4954}&{0.4955}&{0.4956}&{0.4957}&{0.4958}&{0.4959}&{0.4960}\\
      \hline
      GRR$(\theta\epsilon)$&{0.4958}&{0.4975}&{0.4991}&{0.5008}&{0.5025}&{0.5041}&{0.5058}&{0.5074}&{0.5091}&{0.5108}\\
      \hline
      \hline
      Local-map&{0.5024}&{1.0000}&{1.0000}&{1.0000}&{1.0000}&{1.0000}&{1.0000}&{1.0000}&{1.0000}&{1.0000}\\
      \hline
      Adj-map$(\alpha=0.2)$&{0.4312}&{0.5036}&{0.5829}&{0.6565}&{0.7210}&{0.7757}&{0.8213}&{0.8586}&{0.8890}&{0.9133}\\
      \hline
      Adj-map$(\alpha=0.5)$&{0.4185}&{0.5198}&{0.6206}&{0.7071}&{0.7776}&{0.8332}&{0.8763}&{0.9092}&{0.9339}&{0.9522}\\
      \hline
      Adj-map$(\alpha=1)$&{0.4133}&{0.5282}&{0.6378}&{0.7287}&{0.8004}&{0.8551}&{0.8962}&{0.9263}&{0.9482}&{0.9639}\\
      \hline
      Adj-map$(\alpha=2)$&{0.4105}&{0.5332}&{0.6476}&{0.7406}&{0.8126}&{0.8666}&{0.9062}&{0.9348}&{0.9551}&{0.9693}\\
      \hline
      Adj-map$(\alpha=5)$&{0.4087}&{0.5365}&{0.6539}&{0.7482}&{0.8201}&{0.8735}&{0.9122}&{0.9398}&{0.9591}&{0.9724}\\
      \hline
  \end{tabular}
  \setlength{\abovecaptionskip}{-2ex}
  \setlength{\belowcaptionskip}{0ex}
  \caption{The theoretical lower bound $\gamma$ of order-preserving probability for any pair of data points $x_1$ and $x_2$, where $x_1\in \mathcal{P}_i$, $x_2\in \mathcal{P}_j$, $dist=j-i$, $t=|x_1-x_2|$, $|\mathbb{D}_{\bot}|=100$, $\theta=10$, $\epsilon=0.1$.}
  \label{tab:op-prob-t}
\end{table*}

\noindent\textbf{ Order-Preserving Probability Comparisons.}
We traverse all the possible perturbation results that maintain the original order of a pair of values to calculate the exact order-preserving probability for comparison.
To show the gap between the LDP algorithms and the distance-based LDP algorithms in order-preserving capability, we take GRR, which typically satisfies LDP definition, as the object of comparison.
Since GRR satisfies the LDP definition, it's difficult to compare fairly it with our dLDP algorithms under the same $\epsilon$.
Hence we set the privacy budget of GRR as $\epsilon$ and $\theta \epsilon$ respectively.
\begin{figure}[h]
  \centering
  \includegraphics[width=0.35\textwidth]{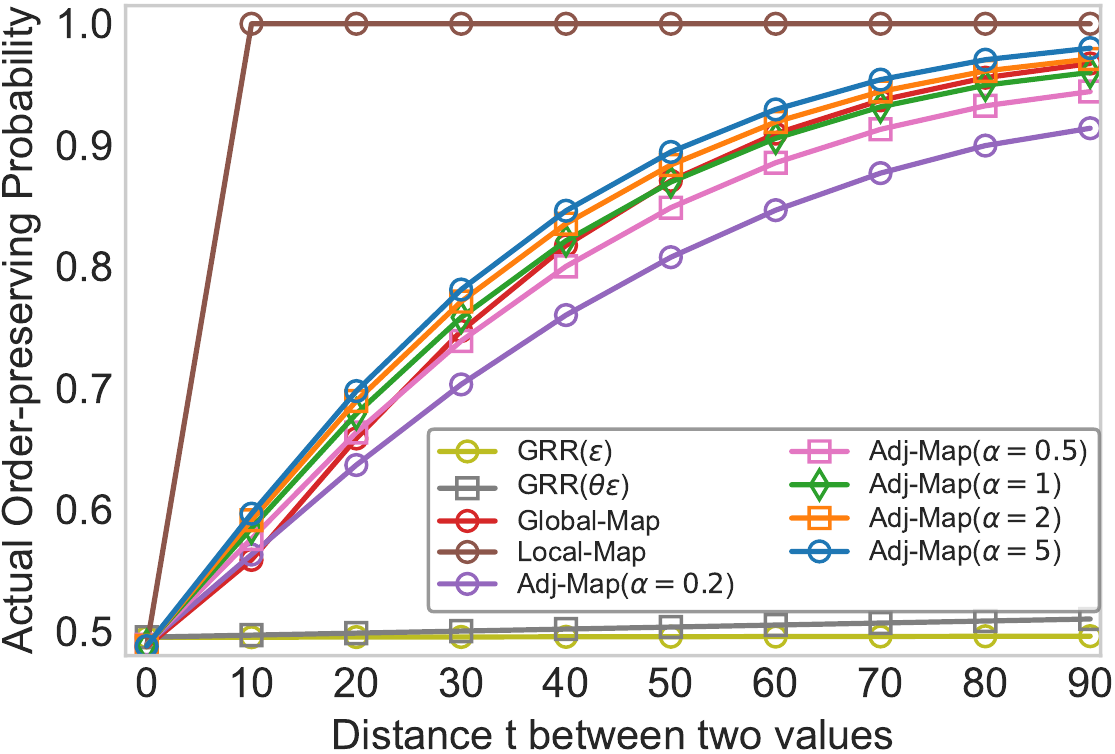}
  \vspace{-2ex}
  \setlength{\belowcaptionskip}{-2ex}
  \caption{Comparison of actual order-preserving probability of algorithms, where $\epsilon=0.1, \theta = 10, |\mathbb{D}_\bot|=100$.}
  \label{fig:order preserving probability}
\end{figure}

\ifshaphered
\emph{(Reviewer2:D3)}
\else
\fi
\revision{
As shown in Figure \ref{fig:order preserving probability}, GRR provides order-preserving probabilities $\gamma$ close to $0.5$ for all value pairs, which is close to randomly shuffling the order of values.
The algorithms proposed in this paper provide the utility with order-preserving probabilities between $0.5$ and $1$.
They make a trade-off between the utility and the privacy of the order-preserving encryption and LDP mechanisms.
}
Among the algorithms proposed in this paper, Local-map gives the highest $\gamma$ because the desensitized partition is deterministic.
When $t\ge\theta$, which means two values fall in different partitions, Local-map guarantees that the order-preserving probability of these two values is 1.
For Adj-map, $\alpha$ is the ratio of $\epsilon_{prt}$ and $\theta\cdot \epsilon_{ner}$.
It's consistent with the analysis in section \ref{sec:algorithms} that the smaller $\alpha$ is, the higher $\gamma$ is.
Furthermore, the lines of Global-map and Adj-map almost coincide when $\alpha=1$, which echoes the result in Figure \ref{fig:distribution}.

\noindent\textbf{ Theoretical Analysis Verification.}
Although we can exhaustively aggregate the accurate order-preserving probability by traversal of all possible output, it's extremely time-consuming when data domain $|\mathbb{D}_{\bot}|$ is large.
To simplify calculations, We deduced the theoretical $\gamma$ in Section \ref{sec:algorithms} and tabulate the calculation results as Table \ref{tab:op-prob-t}.
Note that $\gamma$ may be less than $0.5$.
The reason is that in Definition \ref{def:probabilistic-order-preserving}, $Pr[y_i>y_j] \ge Pr[y_i<y_j] \nRightarrow Pr[y_i>y_j] \ge 0.5$ since $Pr[y_i>y_j]+Pr[y_i<y_j] < 1$ when taking $Pr[y_i=y_j]$ into consideration.
Besides, there is scaling in the derivation of $\gamma$, so the result in Table \ref{tab:op-prob-t} may be less than the accurate values.
In general, the results in Table \ref{tab:op-prob-t} are essentially consist with the accurate aggregation in Figure \ref{fig:order preserving probability}. Thus we can efficiently make approximations based on theoretical derivation without exhaustive calculation.
\subsection{Weighted-Kendall on Synthetic Dataset}
\label{app:figure_kendall}
The results are shown in Figure \ref{fig:syn10_kendall}.

\subsection{The MSE of range query on Synthetic Dataset}
\label{app:figure_range_query}
The results are shown in Figure \ref{fig:rq_uniform}.
\begin{figure}[h]
  \centering
  \includegraphics[width=0.48\textwidth]{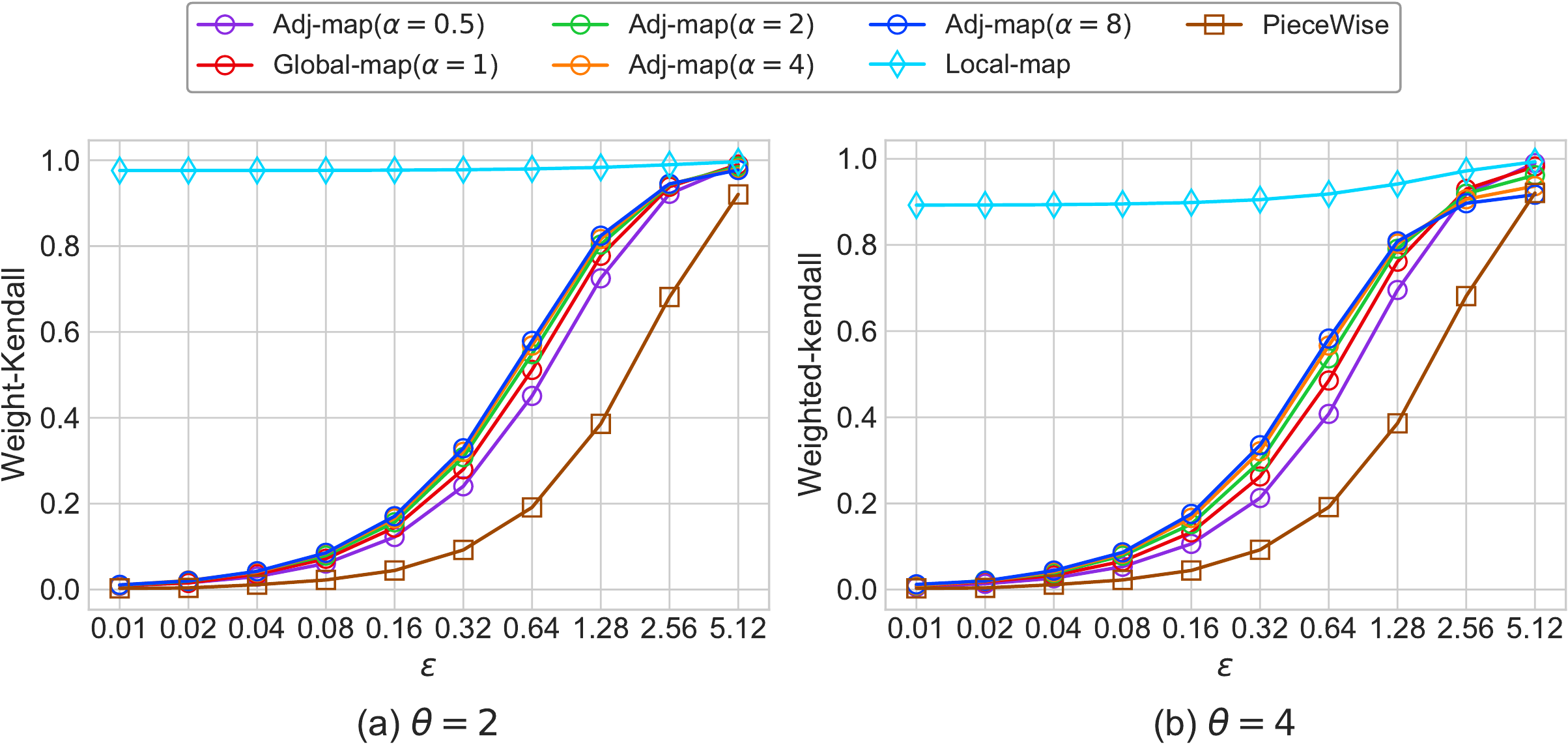}
  \vspace{-5ex}
  \setlength{\belowcaptionskip}{0ex}
  \caption{Weighted-Kendall on Synthetic dataset containing $10k$ uniformly distributed values.}
  \label{fig:syn10_kendall}
\end{figure}
\begin{figure}[h]
  \centering
  \includegraphics[width=0.48\textwidth]{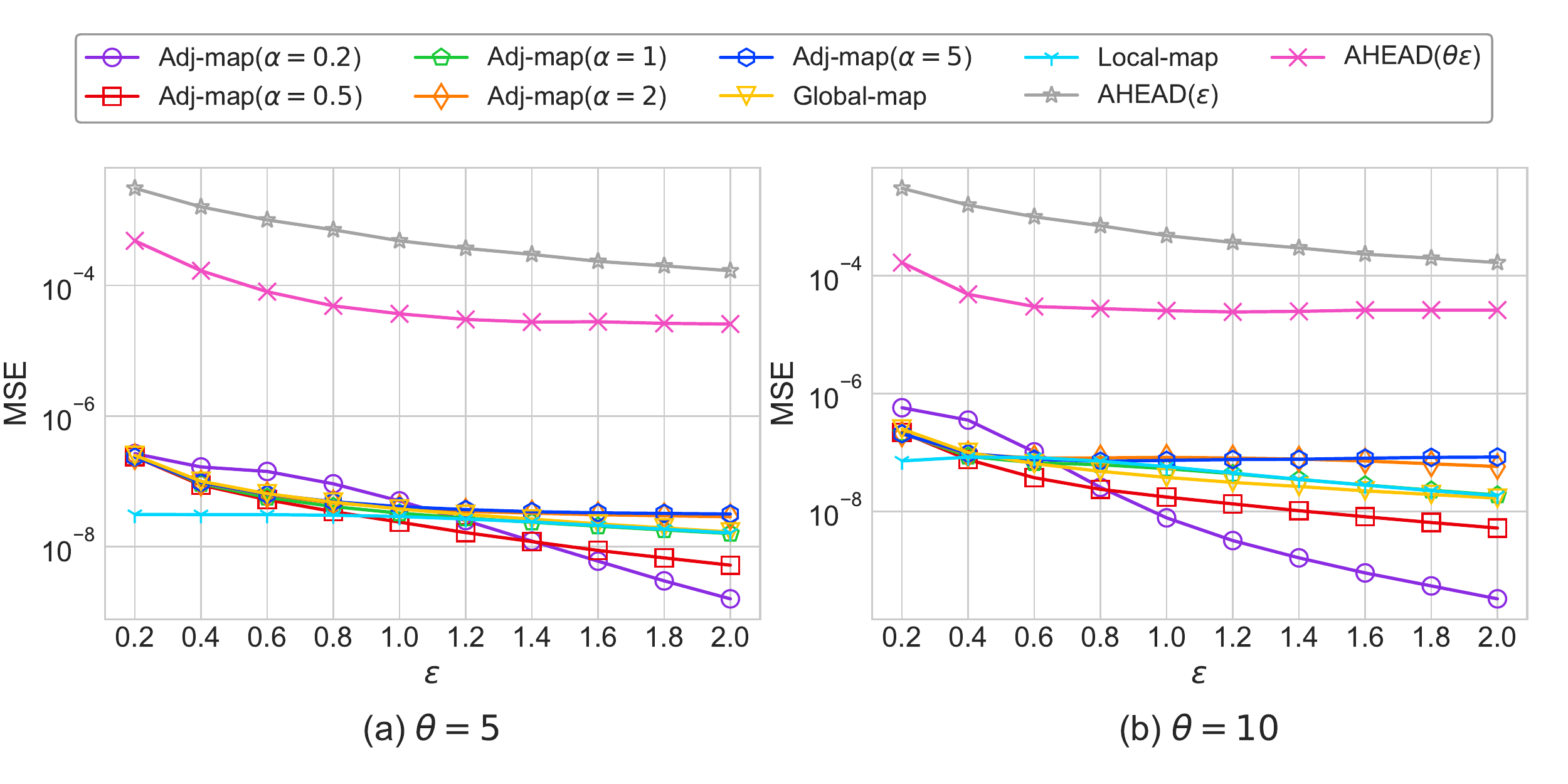}
  \vspace{-5ex}
  \setlength{\belowcaptionskip}{0ex}
  \caption{The MSE of Range Query on Synthetic Dataset.}
  \label{fig:rq_uniform}
\end{figure}
\subsection{Prediction MSE of GBDT Models for Regression Trained on CASP Dataset}
\label{app:figure_casp}
The results are shown in Figure \ref{fig:casp_regress}.
\begin{figure}[h]
  \centering
  \includegraphics[width=0.48\textwidth]{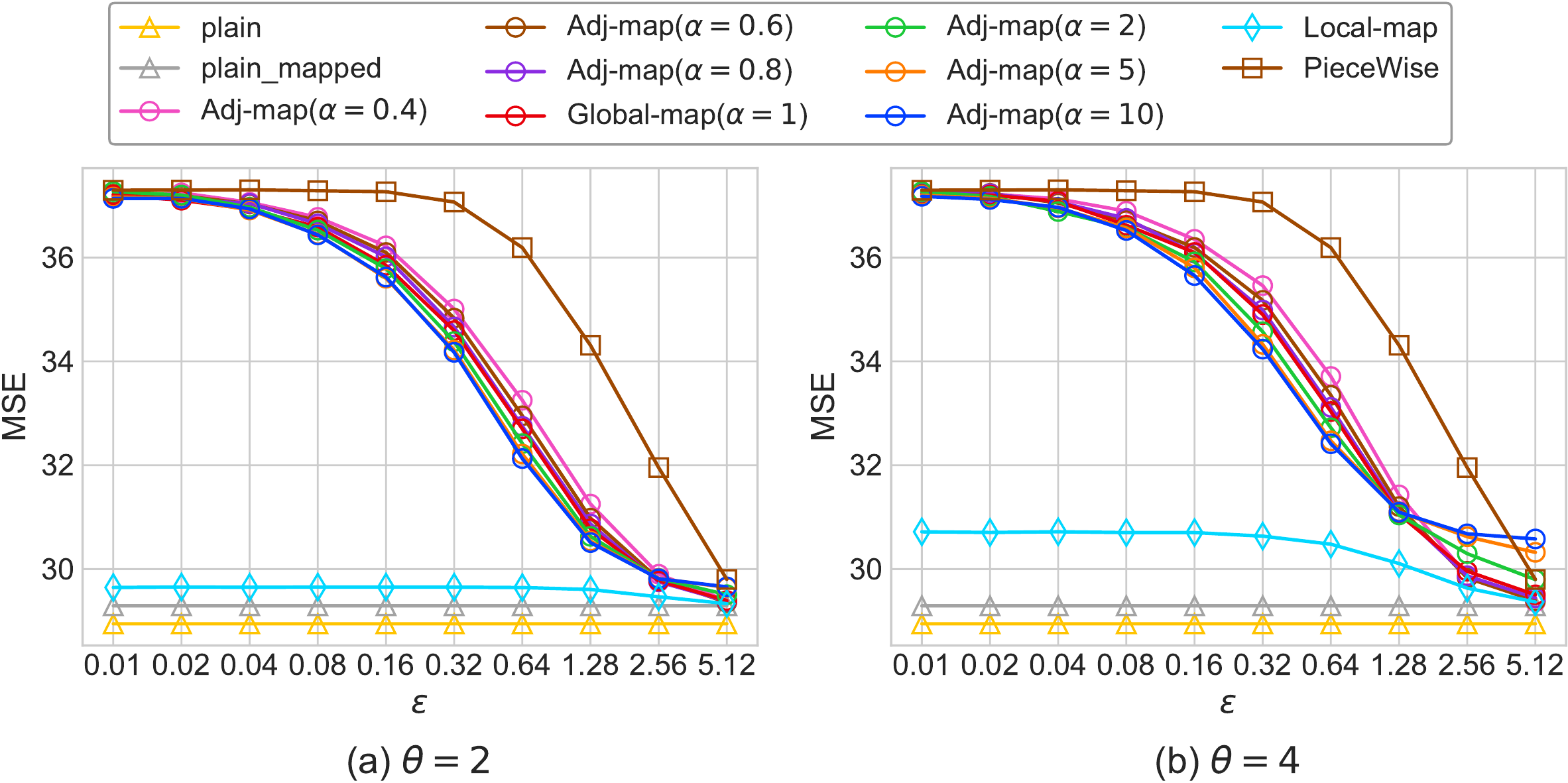}
  \vspace{-5ex}
  \setlength{\belowcaptionskip}{0ex}
  \caption{Prediction MSE of GBDT Models for Regression Trained on CASP Dataset.}
  \label{fig:casp_regress} 
\end{figure}
\subsection{Prediction accuracy of XGBoost models trained by OpBoost}
\label{app:xgboost_result}
The results are shown in Table \ref{tab:opxgboost}.

\begin{table*}[h]
  \small
  \centering
  \begin{tabular}{|c|c|c|c|c|c|c|c|c|c|c|c|c|}
      \hline
      \diagbox{Method}{Task}&\multicolumn{3}{c|}{$\epsilon=0.08,\ \theta=2$}&\multicolumn{3}{c|}{$\epsilon=0.08,\ \theta=4$}&\multicolumn{3}{c|}{$\epsilon=1.28,\ \theta=2$}&\multicolumn{3}{c|}{$\epsilon=1.28,\ \theta=4$}\\
      \cline{2-13}
      &\shortstack{2-Cls.}&\shortstack{M-Cls.}&\shortstack{Reg.}&\shortstack{2-Cls.}&\shortstack{M-Cls.}&\shortstack{Reg.}&\shortstack{2-Cls.}&\shortstack{M-Cls.}&\shortstack{Reg.}&\shortstack{2-Cls.}&\shortstack{M-Cls.}&\shortstack{Reg.}\\
      \hline
      Local-map&{$0.9947\times$}&{$0.9930\times$}&{$0.9208\times$}&{$0.9566\times$}&{$0.9670\times$}&{$0.8187\times$}&{$1.0003\times$}&{$0.9958\times$}&{$0.9361\times$}&{$0.9602\times$}&{$0.9812\times$}&{$0.8690\times$}\\
      \hline
      Adj-map($\alpha=0.4$)&{0.5413$\times$}&{$0.5297\times$}&{$0.6388\times$}&{0.6722$\times$}&{$0.4444\times$}&{$0.6348\times$}&{0.6483$\times$}&{$0.9279\times$}&{$0.8136\times$}&{$0.6433\times$}&{$0.9144\times$}&{$0.8039\times$}\\
      \hline
      Adj-map($\alpha=0.6$)&{0.6343$\times$}&{$0.5393\times$}&{$0.6391\times$}&{$0.7252\times$}&{$0.5266\times$}&{$0.6363\times$}&{0.6398$\times$}&{$0.9380\times$}&{$0.8229\times$}&{$0.6703\times$}&{$0.9249\times$}&{$0.8123\times$}\\
      \hline
      Adj-map($\alpha=0.8$)&{0.6689$\times$}&{$0.5753\times$}&{$0.6415\times$}&{0.6553$\times$}&{$0.4846\times$}&{$0.6389\times$}&{$0.6656\times$}&{$0.9359\times$}&{$0.8277\times$}&{$0.6404\times$}&{$0.9270\times$}&{$0.8140\times$}\\
      \hline
      Global-map($\alpha=1$)&{0.5935$\times$}&{$0.5630\times$}&{$0.6394\times$}&{0.6095$\times$}&{$0.5513\times$}&{$0.6380\times$}&{$0.6797\times$}&{$0.9384\times$}&{$0.8283\times$}&{$0.6300\times$}&{$0.9309\times$}&{$0.8141\times$}\\
      \hline
      Adj-map($\alpha=2$)&{0.7311$\times$}&{$0.5981\times$}&{$0.6401\times$}&{0.6788$\times$}&{$0.5507\times$}&{$0.6414\times$}&{$0.6735\times$}&{$0.9458\times$}&{$0.8321\times$}&{$0.6533\times$}&{$0.9324\times$}&{$0.8123\times$}\\
      \hline
      Adj-map($\alpha=5$)&{0.6751$\times$}&{$0.6177\times$}&{$0.6435\times$}&{0.5870$\times$}&{$0.6106\times$}&{$0.6394\times$}&{$0.7188\times$}&{$0.9478\times$}&{$0.8338\times$}&{$0.8921\times$}&{0.9400$\times$}&{$0.8065\times$}\\
      \hline
      Adj-map($\alpha=10$)&{0.6817$\times$}&{$0.6329\times$}&{$0.6449\times$}&{0.6213$\times$}&{$0.5815\times$}&{$0.6424\times$}&{$0.7733\times$}&{$0.9488\times$}&{$0.8350\times$}&{$0.7554\times$}&{$0.9372\times$}&{$0.8051\times$}\\
      \hline
      Piecewise&{$0.5821\times$}&{$0.1187\times$}&{$0.6249\times$}&{$0.5821\times$}&{$0.1187\times$}&{$0.6249\times$}&{$0.7487\times$}&{$0.7847\times$}&{$0.7014\times$}&{$0.7487\times$}&{0.7847$\times$}&{$0.7014\times$}\\
      \hline
  \end{tabular}
  \caption{Prediction accuracy of XGBoost models trained by OpBoost. We show the ratio of each accuracy to the accuracy of the model trained on the raw dataset. Three kinds of tasks are conducted on Adult, Pen-digits, and CASP datasets, respectively.}
  \label{tab:opxgboost}
\end{table*}

\subsection{Communication and Computation overhead of OpBoost}
\label{app:comun_comp}
The results are shown in Table \ref{tab:comun_perform} and Table \ref{tab:comp_perform}.

\begin{table*}[h]
  \small
  \centering
  \begin{tabular}{|c|c|c|c|c|c|c|c|c|c|c|c|c|}
      \hline
      \diagbox{Method}{Task}&\multicolumn{6}{c|}{Party A}&\multicolumn{6}{c|}{Party B}\\
      \cline{2-13}
      &\multicolumn{3}{c|}{GBDT}&\multicolumn{3}{c|}{XGBoost}&\multicolumn{3}{c|}{GBDT}&\multicolumn{3}{c|}{XGBoost}\\
      \cline{2-13}
      &\shortstack{2-Cls.}&\shortstack{M-Cls.}&\shortstack{Reg.}&\shortstack{2-Cls.}&\shortstack{M-Cls.}&\shortstack{Reg.}&\shortstack{2-Cls.}&\shortstack{M-Cls.}&\shortstack{Reg.}&\shortstack{2-Cls.}&\shortstack{M-Cls.}&\shortstack{Reg.}\\
      \hline
      Local-map&{$74$}&{$349$}&{$96$}&{$198$}&{$768$}&{$385$}&{$521849$}&{$481377$}&{$1318036$}&{$522097$}&{$482215$}&{$1318612$}\\
      \hline
      Global-map($\alpha=1$)&{$26$}&{$637$}&{$165$}&{$203$}&{$695$}&{$401$}&{$521753$}&{$481953$}&{$1318173$}&{$522107$}&{$482069$}&{$1318645$}\\
      \hline
      Adj-map($\alpha=0.4$)&{$26$}&{$740$}&{$178$}&{$224$}&{$820$}&{$459$}&{$521753$}&{$482159$}&{$1318199$}&{$522148$}&{$482321$}&{$1318760$}\\
      \hline
      Adj-map($\alpha=10$)&{$26$}&{$728$}&{$178$}&{$221$}&{$821$}&{$459$}&{$521753$}&{$482136$}&{$1318200$}&{$522143$}&{$482323$}&{$1318760$}\\
      \hline
      Piecewise&{$26$}&{$2313$}&{$253$}&{$500$}&{$12278$}&{$2318$}&{$521753$}&{$485299$}&{$1318348$}&{$522703$}&{$505220$}&{$1322469$}\\
      \hline
  \end{tabular}
  \caption{Total Communication (Bytes) of each Party in OpBoost by using different order-preserving desensitization algorithms with $\epsilon=0.08$, $\theta=4$. Three kinds of tasks are conducted on Adult, Pen-digits, and CASP datasets, respectively.}
  \label{tab:comun_perform}
\end{table*}

\begin{table*}[h]
  \small
  \centering
  \begin{tabular}{|c|c|c|c|c|c|c|c|c|c|c|c|c|}
      \hline
      \diagbox{Method}{Task}&\multicolumn{6}{c|}{Sampling with Bounded DLAP}&\multicolumn{6}{c|}{Sampling with EXP}\\
      \cline{2-13}
      &\multicolumn{3}{c|}{GBDT}&\multicolumn{3}{c|}{XGBoost}&\multicolumn{3}{c|}{GBDT}&\multicolumn{3}{c|}{XGBoost}\\
      \cline{2-13}
      &\shortstack{2-Cls.}&\shortstack{M-Cls.}&\shortstack{Reg.}&\shortstack{2-Cls.}&\shortstack{M-Cls.}&\shortstack{Reg.}&\shortstack{2-Cls.}&\shortstack{M-Cls.}&\shortstack{Reg.}&\shortstack{2-Cls.}&\shortstack{M-Cls.}&\shortstack{Reg.}\\
      \hline
      Local-map&{$1655.1$}&{$2926.5$}&{$2697.8$}&{$951.7$}&{$1089.7$}&{$2151.1$}&{$913.8$}&{$2284.8$}&{$826.1$}&{$218.1$}&{$424.6$}&{$325.7$}\\
      \hline
      Global-map($\alpha=1$)&{$1154.4$}&{$2484.6$}&{$1556.5$}&{$462.3$}&{$1587.8$}&{$1010.7$}&{$904.0$}&{$2253.5$}&{$838.7$}&{$214.1$}&{$428.4$}&{$316.8$}\\
      \hline
      Adj-map($\alpha=0.4$)&{$2187.3$}&{$3382.7$}&{$4070.2$}&{$1476.9$}&{$1587.8$}&{$3481.0$}&{$952.8$}&{$2248.6$}&{$968.2$}&{$268.4$}&{$477.7$}&{$448.2$}\\
      \hline
      Adj-map($\alpha=10$)&{$9135.1$}&{$10052.8$}&{$21504.3$}&{$8313.3$}&{$8054.4$}&{$20584.5$}&{$953.8$}&{$2287.7$}&{$975.7$}&{$272.5$}&{$481.7$}&{$447.8$}\\
      \hline
      Piecewise&{$958.5$}&{$2163.8$}&{$953.1$}&{$305.5$}&{$650.4$}&{$495.1$}&{$-$}&{$-$}&{$-$}&{$-$}&{$-$}&{$-$}\\
      \hline
  \end{tabular}
  \caption{Run time (ms) of the entire training process of OpBoost by using different order-preserving desensitization algorithms with $\epsilon=0.08$, $\theta=4$. Three kinds of tasks are conducted on Adult, Pen-digits, and CASP datasets, respectively.}
  \label{tab:comp_perform}
\end{table*}
\end{appendix}

\end{document}